\documentclass{article}

\usepackage{microtype}
\usepackage{graphicx}
\usepackage{subfigure}
\usepackage{booktabs} % for professional tables
\usepackage{caption, subcaption}
\usepackage{algorithmic,algorithm}
\usepackage{hyperref}
\usepackage{smile}

\usepackage[accepted]{icml2024}

\usepackage{amsmath}
\usepackage{amssymb}
\usepackage{mathtools}
\usepackage{amsthm}

\usepackage{xcolor}
\usepackage{multirow}
\usepackage{colortbl}

\usepackage[capitalize,noabbrev]{cleveref}

\theoremstyle{plain}
\newtheorem{theorem}{Theorem}[section]
\newtheorem{proposition}[theorem]{Proposition}

\theoremstyle{definition}

\theoremstyle{remark}
\newtheorem{remark}[theorem]{Remark}

\usepackage[textsize=tiny]{todonotes}

\icmltitlerunning{Referee Can Play: Conditional Generation Via Model Inversion}

\begin{document}

\twocolumn[
\icmltitle{Referee Can Play: An Alternative Approach to Conditional Generation via Model Inversion}

\icmlsetsymbol{equal}{*}

\begin{icmlauthorlist}
\icmlauthor{Xuantong Liu}{sch}
\icmlauthor{Tianyang Hu}{comp}
\icmlauthor{Wenjia Wang}{sch}
\icmlauthor{Kenji Kawaguchi}{sch1}
\icmlauthor{Yuan Yao}{sch}
\end{icmlauthorlist}

\icmlaffiliation{comp}{Huawei Noah's Ark Lab}
\icmlaffiliation{sch}{The Hong Kong University of Science and Technology}
\icmlaffiliation{sch1}{National University of Singapore}

\icmlcorrespondingauthor{Tianyang Hu}{hutianyang.up@outlook.com}

\icmlkeywords{Machine Learning, ICML}

\vskip 0.3in
]

\printAffiliationsAndNotice{}  % leave blank if no need to mention equal contribution

\begin{figure*}[ht]
\setcounter{subfigure}{0}
\begin{center}
    \subfigure[\textbf{Our method}]{
        \begin{minipage}[t]{0.15\linewidth}
            \centering
            \includegraphics[width=1\linewidth]{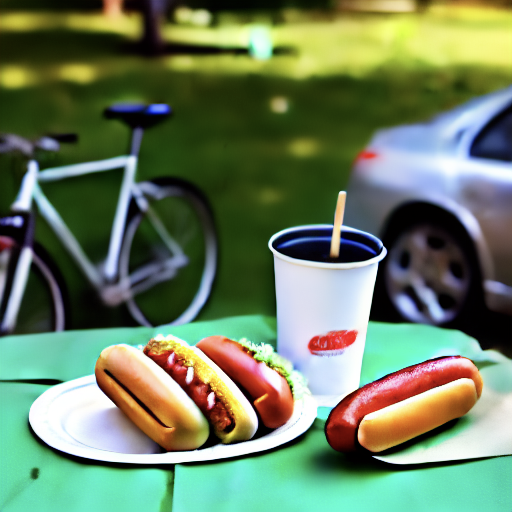}
        \end{minipage}
    }
    \subfigure[SD v1.5]{
        \begin{minipage}[t]{0.15\linewidth}
            \centering
            \includegraphics[width=1\linewidth]{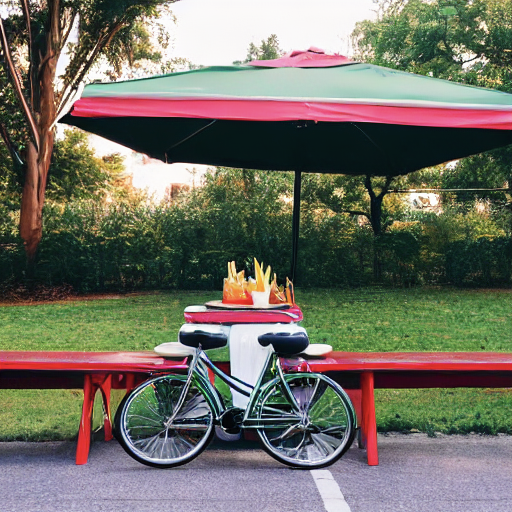}
        \end{minipage}
        }
    \subfigure[Attn-Exct]{
        \begin{minipage}[t]{0.15\linewidth}
            \centering
            \includegraphics[width=1\linewidth]{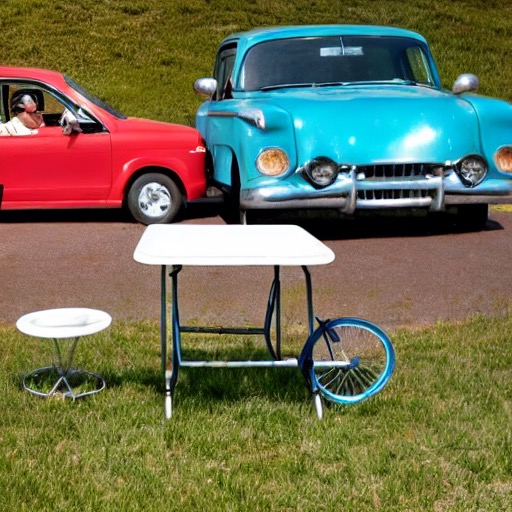}
        \end{minipage}
        }
    \subfigure[PixArt-$\alpha$]{
        \begin{minipage}[t]{0.15\linewidth}
            \centering
            \includegraphics[width=1\linewidth]{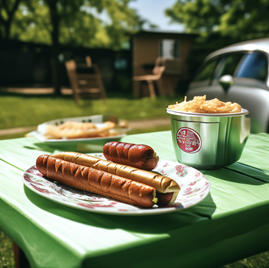}
        \end{minipage}
        }
    \subfigure[DALLE-2]{
        \begin{minipage}[t]{0.15\linewidth}
            \centering
            \includegraphics[width=1\linewidth]{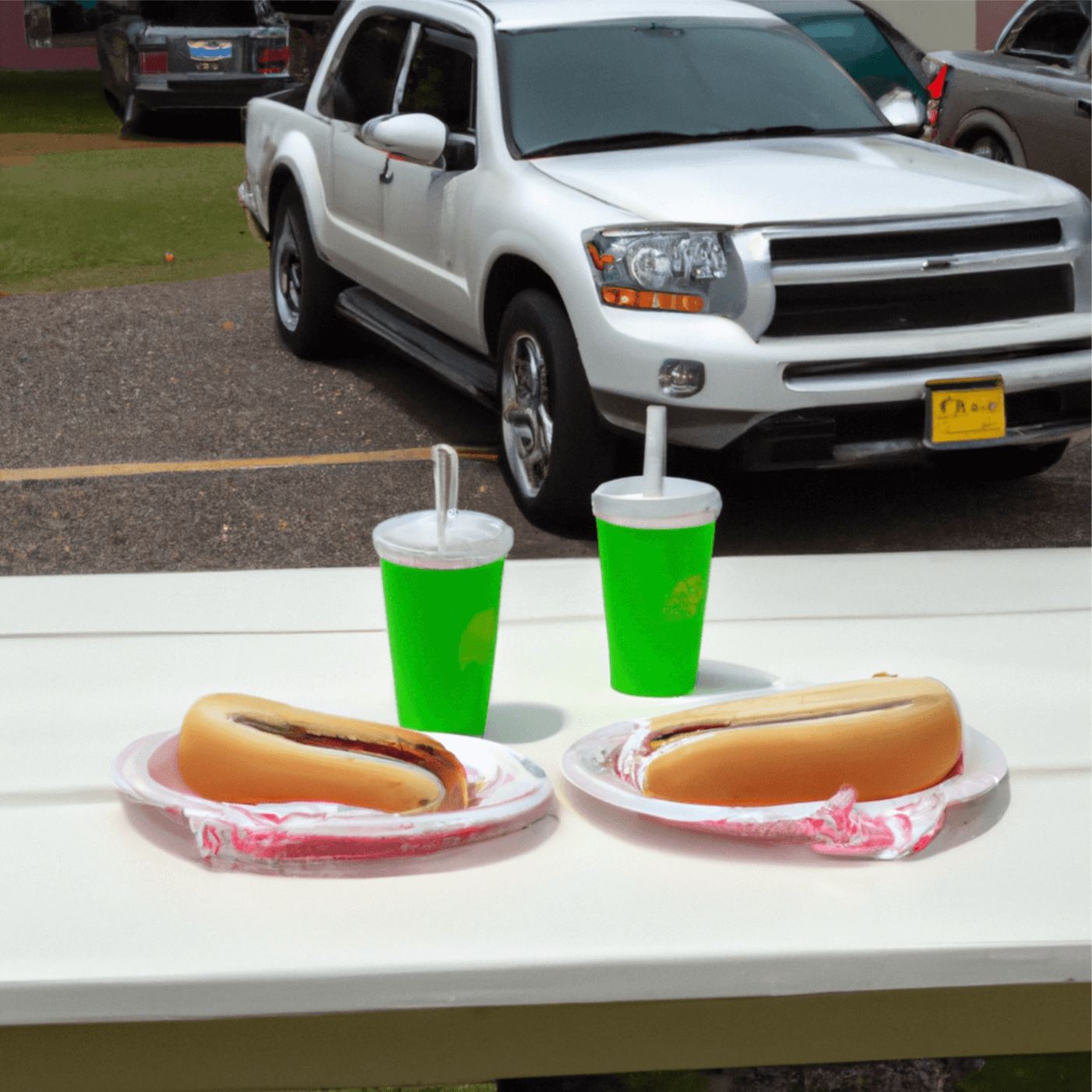}
        \end{minipage}
        }
    \subfigure[DALLE-3]{
        \begin{minipage}[t]{0.15\linewidth}
            \centering
            \includegraphics[width=1\linewidth]{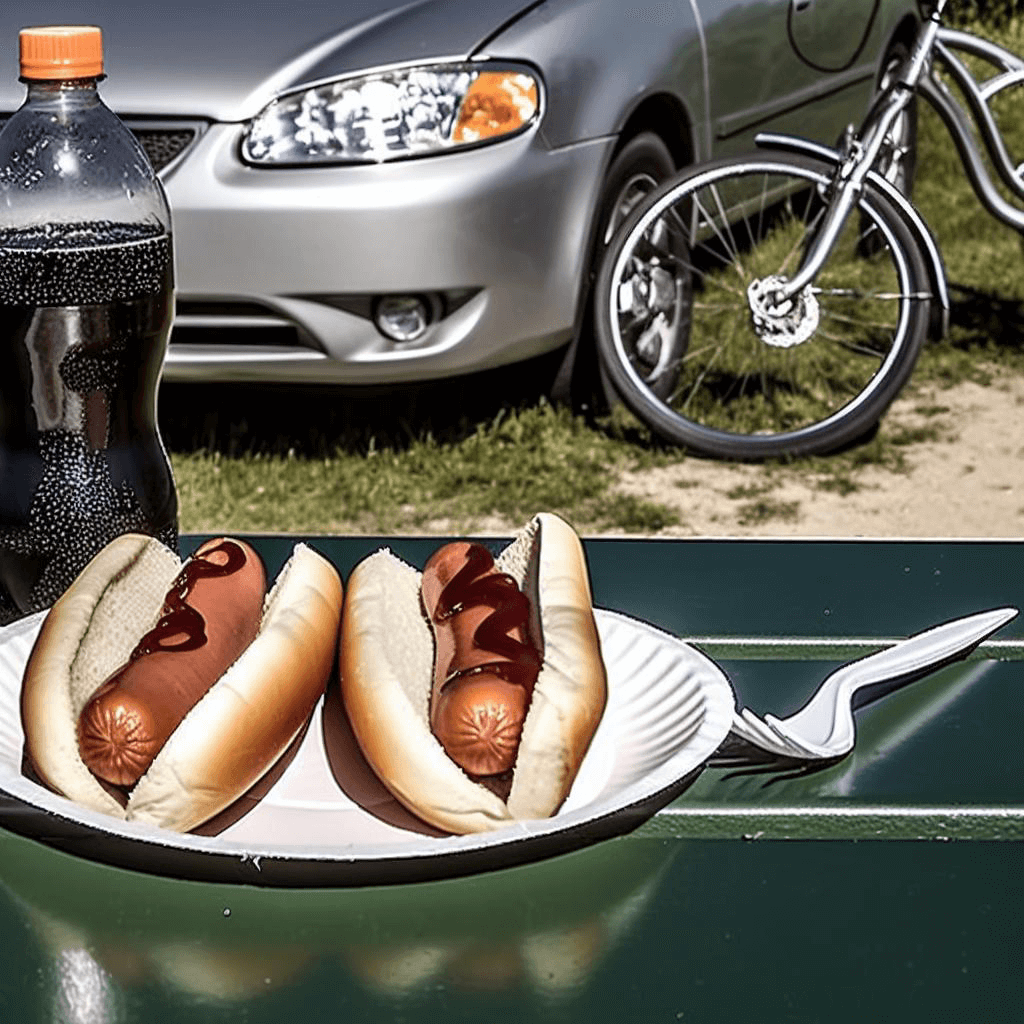}
        \end{minipage}
        }
\end{center}
\label{fig1}
\caption{Our method can effectively generate faithful images strictly following the prompt ``Two hot dogs sit on a white paper plate near a soda cup which is sitting on a green picnic table while a bike and a silver car are parked nearby”. However, the baseline methods, including Stable Diffusion (SD) 1.5, Attend-and-Excite \citep{chefer2023attendandexcite}, PixArt-$\alpha$ \citep{chen2023pixart}, DALLE-2 \citep{dalle2}, DALLE-3 \citep{betker2023dalle3}, struggle to generate the right images encountering this kind of complex compositional prompts.}
\end{figure*}

\begin{abstract}
As a dominant force in text-to-image generation tasks,
Diffusion Probabilistic Models (DPMs) face a critical challenge in controllability, struggling to adhere strictly to complex, multi-faceted instructions. 
In this work, we aim to address this alignment challenge for conditional generation tasks. 
First, we provide an alternative view of state-of-the-art DPMs as a way of inverting advanced Vision-Language Models (VLMs). 
With this formulation, we naturally propose a training-free approach that bypasses the conventional sampling process associated with DPMs. 
By directly optimizing images with the supervision of discriminative VLMs, the proposed method can potentially achieve a better text-image alignment. 
As proof of concept, we demonstrate the pipeline with the pre-trained BLIP-2 model and identify several key designs for improved image generation. 
To further enhance the image fidelity, a Score Distillation Sampling module of Stable Diffusion is incorporated. 
By carefully balancing the two components during optimization, our method can produce high-quality images with near state-of-the-art performance on T2I-Compbench.
\end{abstract}

\section{Introduction}
With exceptional sample quality and scalability, 
DPMs \citep{sohl2015deep, ho2020denoising, song2020score} have significantly contributed to the success of Artificial Intelligence Generated Content (AIGC), especially in text-to-image generation. 
The performance of state-of-the-art (SOTA) DPMs, e.g., Stable Diffusion \citep{rombach2022high, podell2023sdxl}, PixArt-$\alpha$ \citep{chen2023pixart}, in generating various images with \textit{high-fidelity} is no longer a major concern.
What is left to be desired is \textit{controllability}, i.e., the compatibility between the generated image and the input text. 
For instance, a simple composite prompt like ``a red backpack and a blue book" is challenging for Stable Diffusions (SD) \citep{huang2023t2i}.

% The problem is the condition injection. 
To better understand the condition injection mechanism of modern text-to-image generation models, let us first consider the relationship between text (i.e., condition) and image.
In practice, the actual condition the model takes usually contains an extended version of the input prompt paraphrased by language models. 
As the text description gets richer, more information is dictated for the target image with less nuisance to fill. 
In other words, the stochasticity of the generation task gradually decreases with extra conditions, where we may even expect almost one-to-one matching between the text and image. 
This transition calls for new thinking about the conditional generation task. 

% transition from generation to model inversion
On the one hand, when training such strong-condition models, high-quality paired data $(\xb, \yb)$ are needed where $\xb$ satisfies the condition $\yb$. 
For most of the generative models, the generation process can be described as mapping $(\yb, \mathbf{\epsilon})$ to $\xb$ where $\mathbf{\epsilon}$ is a random vector providing diversity. 
In practice, the controllability of the generative model is critically dependent on the label quality, i.e., $\yb|\xb$ should be as detailed as possible. 
Beyond human labeling, SOTA DPMs such as DALLE-3 \citep{betker2023dalle3} utilize powerful vision language models (VLMs) to regenerate image captions during training. 
Then, given a new prompt $\yb$ at the inference phase, the ideal image $\xb|\yb$ should be deemed fit by the VLM. 
From this perspective, the text-to-image generation case can be seen as a \textit{model inversion} task on the VLM, with explicit mappings parametrized by the score network in DPMs.
% A good discriminative model or conditional guidance is key 
On the other hand, the discriminative module also plays a more central role in aligning with the condition.
As supporting evidence, consider the importance of guidance in current text-to-image DPMs \citep{ho2022classifier, bansal2023universal, ma2023elucidating}. 
Although designed to work directly, a relatively large classifier-free guidance (CFG) coefficient is indispensable (7.5 by default in SD). 
When faced with the strong condition of a multi-attribute object correspondence, the image produced by a higher CFG is more semantically compliant (see results in Appendix \ref{apdix:discri_gener}).

% The lack of alignment may not be due to the generative model part, but rather the guidance part
The utilization of VLM in DALLE-3 and a higher CFG in SD emphasize the importance of the discriminative component for image-text alignment, which may have been overshadowed by the generative counterpart. 
Fortunately, discrimination (image-to-text) is relatively easier than generation (text-to-image), since an image often contains more information than its text description.  
When a misalignment happens during diffusion generation, it can be trivial for a VLM to tell. 
Based on this premise, T2I-CompBench \citep{huang2023t2i} employs several discriminative VLMs, e.g., BLIP-VQA \cite{li2022blip} and CLIP \cite{radford2021clip} as referees, to measure text-image compatibility of SOTA generative models. The alignment can be further improved by finetuning with ``good" text-image data selected by the referees.

Inspired by the above observations, we propose a novel conditional image generation paradigm predominantly led by VLMs for improved image-text alignment. 
Given textual prompts, we generate images by conducting model inversion on the VLM. Specifically, we keep the VLM parameters fixed, treating the image as the optimization object. 
For more efficient parameterization, we adopt the pre-trained VAE used in the latent diffusion model \citep{rombach2022high} and conduct optimization on the latent variable space $\zb$.
We start with random noise and progressively refine $\zb$ by minimizing the loss function used in VLM pre-training, which measures image-text consistency (\textbf{Figure} \ref{fig:method} illustrates the core idea of our method and its comparison against current DPMs).
Additionally, to enhance image fidelity, we propose incorporating gradients provided by Score Distillation Sampling. 
The overall generation pipeline is entirely training-free and data-free, and the controllability of our method is greatly improved compared to SD.
While our method may not generate images as elaborate as those specialized generation models like DALLE-3, we achieve comparable performance on image-text alignment. 

\begin{figure}[t]
    \includegraphics[width=0.48\textwidth]{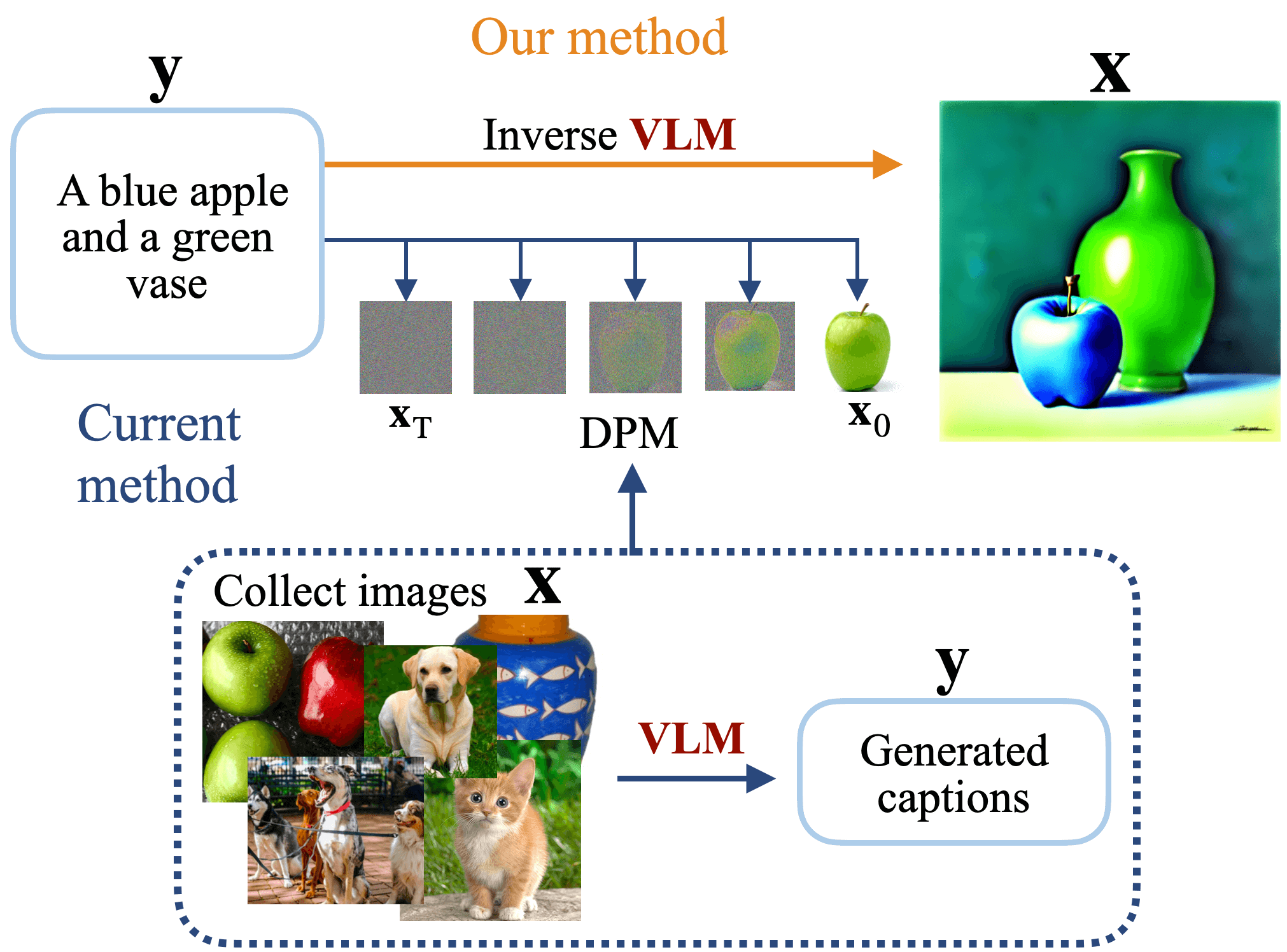}
    \caption{SOTA image generation models require collecting a large number of images which are then labeled in detail by a VLM for training. During the inference phase, the condition ($\yb$) is injected into the denoising process of random noise ($\xb_\mathrm{T}$) through a cross-attention mechanism. The key difference in our approach is that we generate images by directly inversing the VLM, eliminating the need to train a specialized generation model.}
    \label{fig:method}
\end{figure}

The main contributions of this work are summarized below. 
\begin{itemize}
 % Our novel formulation. 
    \item 
    We introduce a novel perspective by understanding strong conditional text-to-image generation as \textit{model inversion}, shedding light on the crucial role of discriminative models in the conditional generation process.
    \item % Our novel method
    We propose a method that places discriminative models such as VLMs at the forefront, introducing a shift in the text-to-image generation paradigm. 
    Our method is training-free and highly flexible. Several key design choices are elucidated, e.g., augmentation regularization, exponential moving average with restart, etc. 
    \item % Empirical results + interesting findings. 
    Our method achieves near SOTA results on benchmarks that measure the generation models' controllability \citep{huang2023t2i}.
\end{itemize}

\section{Preliminary}
\paragraph{Discriminative Vision Language Models}
VLMs represent a pivotal advancement in the field of multimodal representation learning. These models undergo extensive pre-training on expansive datasets comprising image-text pairs, enabling them the capacity for zero-shot predictions across a diverse spectrum of tasks that encompass both visual and linguistic information, including image-text retrieval, image captioning, visual question answering (VQA), etc \citep{du2022vlmsurvey}. The central objective of VLM pre-training is to imbue these models with a profound understanding of the intricate alignment between textual and visual modalities. To achieve this goal, various alignment losses are employed, e.g., contrastive loss in CLIP \citep{radford2021clip}, ALIGN \citep{jia2021ALIGN}, ALBEF \citep{li2021align}, BLIP \citep{li2022blip}, BLIP -2\citep{li2023blip2}. 
In contrast to text-to-image generation models, VLMs exhibit a superior aptitude for aligning textual and visual information. CLIP image-text similarity is widely utilized to measure the alignment of the generated images with the textual condition \citep{hessel2021clipscore, ma2023elucidating}. A recent compositional text-to-image generation evaluation benchmark T2I-CompBench \citep{huang2023t2i} further uses the BLIP-VQA model \citep{li2022blip} as a referee to judge the correctness of the generated images.

\paragraph{Latent Diffusion Models } 
Latent diffusion models (LDMs) \citep{rombach2022high} conduct the forward and reverse processes in the latent space of an autoencoder. The additional encoder $\mathrm{ENC}$ and decoder $\mathrm{DEC}$ are required to map the original image $\xb$ to a latent variable $\mathbf{z}$ and reconstruct the image from $\zb$, such that $\mathrm{DEC}(\mathrm{ENC}(\mathbf{x}))\approx \mathbf{x}$.
Classifier-free guidance (CFG) \citep{ho2022classifier} is a commonly used conditional generation method in DPMs. Given a condition $\yb$ and a pre-trained text-to-image DPM with the noise prediction neural network $\phi$, CFG generates images via $\hat{\epsilon}(\mathbf{x}_t;\yb,t)=\epsilon_{\phi}(\mathbf{x}_t;\yb,t)+s(\epsilon_{\phi}(\mathbf{x}_t;\yb,t)-\epsilon_{\phi}(\mathbf{x}_t;t)$), where $s>0$ is the guidance scale. 

\paragraph{Score Distillation } 
Score distillation sampling (SDS) is an optimization mechanism to distill the rich knowledge from pre-trained text-to-image generation diffusion models \citep{poole2022dreamfusion, luo2023diff}. SDS allows optimizing differentiable generators, and it has been widely explored in text-to-3D generation \citep{wang2023prolificdreamer, wang2023steindreamer}, and image editing tasks \citep{hertz2023delta, kim2023collaborative}. 
Given a pre-trained text-to-image LDM with the noise prediction neural network $\phi$ with noise $\epsilon \sim
\mathcal{N}(\mathbf{0}, \mathbf{1})$, SDS
optimizes a group of parameters $\theta$ by:
\begin{equation*}
    \nabla_{\theta}\mathcal{L}_{SDS}(\phi, \xb=g(\theta)) \approx \mathbb{E}_{t,\mathbf{\epsilon}}[(\mathbf{\epsilon_\phi}(\zb_t;\yb,t)-\mathbf{\epsilon})\frac{\partial \xb}{\partial \theta}],
\end{equation*}
where $g(\theta)$ can be any differentiable function. 
% $g(\theta)=\theta$ in the above equation allows directly optimizing a single image.

\section{Conditional generation via model inversion}

The image caption quality is critical for training text-to-image generation models that possess good image-text understanding ability.
State-of-the-art models have increasingly harnessed discriminative VLMs' capabilities to improve the controllability and precision of image synthesis, such as DALLE-3 \citep{betker2023dalle3} and PixArt-$\alpha$ \citep{chen2023pixart}.
By generating comprehensive and detailed textual descriptions for images ($\xb \stackrel{\mathrm{VLM}}{\longrightarrow} \yb$), the gap between textual and visual information narrows down, allowing the image generation model (e.g., DPMs) to master a more exact alignment. 
At the inference stage, the model is required to generate images from the provided prompt ($\yb \stackrel{\mathrm{DPM}}{\longrightarrow} \xb$). Therefore, the diffusion model essentially functions as a learned \textit{inverse} of the VLM.

Instead of training a DPM to learn the inverse function of VLM, we venture into directly leveraging VLM to perform the reverse task: generating the corresponding image given a text prompt via optimization. 
From the perspective of model inversion, the inverse map learned by mainstream DPMs is parameterized by the score net and can be thought of as continuous (it's a deterministic map if using ODE sampler). 
In clear contrast, the inverse map from direct optimization-based model inversion is stochastic and not continuous. 
Thanks to the discontinuity, the inverse map from the latter approach is not limited by any neural network parameterization and can be unlimited in terms of expressivity \citep{feng2021uncertainty, hu2023complexity}.

\subsection{Vision Language Model Inversion for Alignment}

To fully harness the text and image alignment information ingrained within the pre-trained VLMs, we propose to synthesize text-conditional images commencing from random noise with the supervision of a VLM.  

\paragraph{Problem formulation}
Denote the alignment loss given by a VLM as $L(\xb, \yb)$, where a smaller $L(\xb, \yb)$ indicates better alignment between $\xb$ and $\yb$. We formulate the conditional generation task as an optimization problem, i.e., $\min_{\xb} L(\xb, \yb)$,
where the VLM is kept frozen, and the image is the optimizing target. 
No extra training or data is needed, so our method is highly efficient.
As a proof of concept, we demonstrate with the pre-trained BLIP-2 \citep{li2023blip2} model, which has lightweight training loss and good zero-shot performance. 
We utilize the image-text matching loss $L_{itm}$ and the caption generation loss $L_{cg}$ in BLIP-2 pre-training~\citep{li2021align}, i.e., $L(\xb, \yb) = L_{itm} +L_{cg}$.

\paragraph{Parameterization} 
Images usually reside in high-dimensional spaces (e.g., $512\times 512\times 3$), which makes the optimization problem challenging and computationally expensive. Inspired by LDMs \citep{rombach2022high}, we utilize a pre-trained autoencoder to reduce the dimensionality of the optimization target from pixel space to low-dimensional latent space. We choose the widely used KL-VAE with a downsampling factor $f=8$ in the LDM \citep{rombach2022high}. Denote the decoder as $\mathrm{DEC}$.
Then, $\xb$ is parameterized as $\mathrm{DEC}(\zb)$ and we can directly work with $\zb \in \mathbb{R}^{64\times 64 \times 4}$.

\paragraph{Necessity of Regularization}
Achieving high-quality images through direct discriminative model inversion is almost impossible \citep{yin2020deepinversion, wang2022traditionalCaG}.
The first row of \textbf{Figure} \ref{fig:augs_result} shows the generated images from a naive implementation of the above proposal. During optimization, the alignment loss can be effectively minimized, but the resulting images are not natural. 
They can be thought of as \textit{adversarial} samples that are recognized by VLMs but not humans. 
The existence of adversarial samples has been extensively studied in the adversarial robustness literature \citep{goodfellow2014explaining, arrieta2020explainable}. 
% Emphasize the main intuition: Lower effective search dimension
Therefore, generating plausible images by vanilla optimization is not viable. 
Extra regularization is needed to constrain the search space to align with human perception.

To address the challenge, we take inspiration from contrastive learning, where semantic invariant augmentations are constructed to specify the equivariant groups in the feature space so that learning can be more efficient \citep{dangovski2021equivariant, haochen2022theoretical, hu2022your}. 
Similarly, we consider generating multiple samples with semantic invariance but slight image variations through data augmentation $A\sim\mathbb{P}_{A}$ and use their averaged loss as the optimization objective.
Correspondingly, we can define the \textit{augmentation-regularized loss} as 
\begin{equation}
\label{eqn:aug_loss}
        \tilde{L}_{\yb}(\xb) :=\EE_A L(A(\xb), \yb),
\end{equation}
where the expectation is taken over the random augmentations. 
This is similar to random smoothing \citep{cohen2019certified, li2019certified, ding2023random}. 
The augmentation regularizes the search space by removing the adversarial solutions $\xb$ where $L(\xb)$ is low while $\tilde{L}(\xb)$ is high. 
To illustrate, consider $\mathbb{P}_A$ as random masking. Then, for any specific masking $A'$,  
$\tilde{L}_{\yb}(\xb) \le c$ implies that 
\begin{align*}
    \tilde{L}_{\yb}(A'(\xb)) \le & c + \mathbb{P}(A \notin S_{A'})\cdot\\
    &  \EE_{A \notin S_{A'}}\left[L_{\yb}((A\circ A')(\xb)) - L_{\yb}(A(\xb))\right],
\end{align*}
where $S_{A'}=\{A \mid A\circ A' = A\}$. Thus, minimizing $\tilde{L}(\xb)$ implicitly minimizes $\tilde{L}(A'(\xb))$ for all $A'$ such that $\mathbb{P}(A \notin S_{A'})$ is low. 
To verify, an image $\xb$ in the 1st row of Figure \ref{fig:augs_result} indeed obtains a low $L(\xb)$, while having a high $\tilde{L}(\xb)$ (see results in \textbf{Table} \ref{tab:aug_blip_loss}). Thus, such an adversarial solution is removed in the 2nd and 3rd rows where a returned image $\xb$ is ensured to have low values in $\tilde{L}$. 
This augmentation regularization has been proven effective in improving overall results, including the perceptual image quality (\textbf{Figure} \ref{fig:augs_result}) and BLIP-VQA score (\textbf{Table} \ref{tab:blip_aug_vqa}).

\begin{figure}[t]
    \centering
    \includegraphics[width=0.47\textwidth]{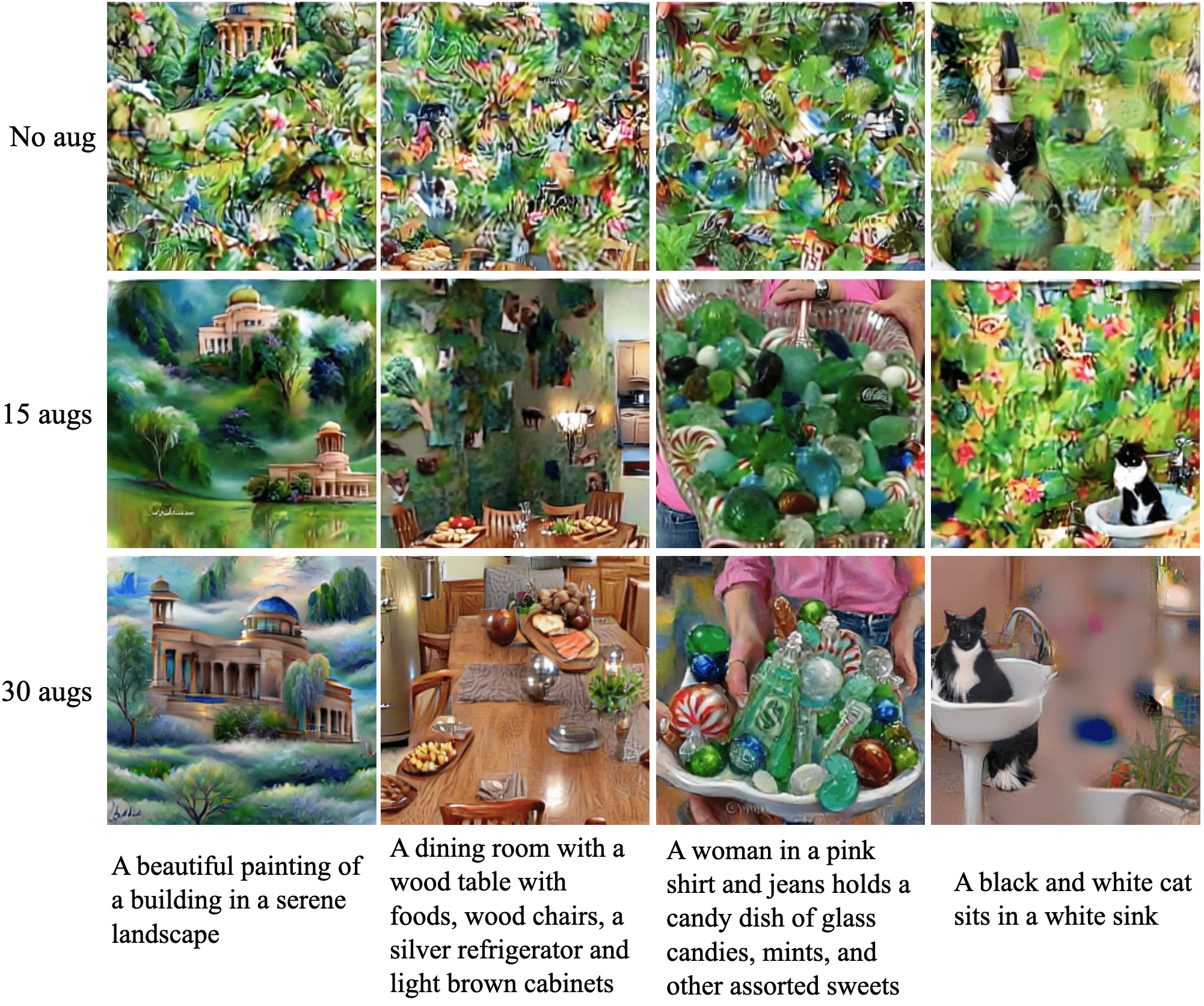}
    \caption{The illustration of the influence of augmentation regularization on BLIP-2 inversion. We show the result of no augmentation (top row), 15 augmentations (middle row), and 30 augmentations (bottom row).}
    \label{fig:augs_result}
\end{figure}

\begin{table}[H]
\centering
\begin{small}
\caption{The BLIP-VQA score of BLIP-2 inversion without and with augmentations on \textit{Attribute-binding} dataset}
\label{tab:blip_aug_vqa}
\begin{tabular}{l|ccc}
\toprule
 & \textbf{color}  & \textbf{shape}  & \textbf{texture} \\ \midrule
no aug & 0.6561 & 0.5078 & 0.5371 \\ 
30 augs & 0.8639 & 0.6686 & 0.7311  \\ \bottomrule
\end{tabular} 
\end{small}
\end{table}

\begin{table}[H]
\centering
\begin{small}
\caption{The average BLIP-2 loss ($L_{itm} \& L_{cg})$ of images on the 1st and 3rd row of \textbf{Figure} \ref{fig:augs_result}.}
\label{tab:aug_blip_loss}
\vspace{0.1in}
\begin{tabular}{l|cc}
\toprule
 & $L(\xb)$  & $\tilde{L}(\xb)$ \\ \midrule
no aug (1st row) & 2e-3 $\&$ 3.28 & 5.65 $\&$ 4.66 \\ 
30 augs (3rd row) & 1e-5 $\&$ 3.08 & 2e-3 $\&$ 3.30  \\ \bottomrule
\end{tabular} 
\end{small}
\end{table}

As an added benefit, augmentation regularization contributes to smoother images and lower loss values, as can be seen in \textbf{Table} \ref{tab:aug_blip_loss}.
The following propositions unveil that Gaussian random augmentation will result in a smoother objective function and easier optimization for model inversion. All technical details can be found in \textbf{Appendix} \ref{sec:app_tech}.
\begin{proposition}[informal]
    Under some mild regularity conditions on the augmentation,  $\tilde{L}_y$ is strictly smoother than $L_y$. Particularly, if $A(\xb) = \xb+\epsilon$ where $\epsilon$ is Gaussian,  $\tilde{L}_y$ has infinite smoothness and is Lipschitz continuous. 
    % and is $\sqrt{\frac{2 R^2}{\pi \sigma^2}}$-Lipschitz when $\epsilon \sim \mathcal{N}(0, \sigma^2 I)$ and $L_y: \RR^d \mapsto \RR$ with $R\ge \sup_{x \in \RR^d} |L_y(x)|$.
\end{proposition}
The optimization problem can greatly benefit from both a smoother objective function \citep{kovalev2022first} and bounded Lipschitz constant
\citep{ghadimi2013stochastic, bertsekas2016nonlinear}. 
From another angle, augmentation can also improve the condition number of convex optimization, as stated in the following proposition. 
\begin{proposition}[informal]
    Assume $L_c(\cdot)$ is a convex function and denote its Hessian matrix at $x$ as $\partial^2 L_c(x)$. Under some mild regularity conditions, the condition number of $\partial^2\tilde{L}_c$ is strictly smaller than that of $\partial^2 L_c(x)$. 
\end{proposition}

\paragraph{Ideal Augmentations for BLIP-2 Inversion} 
Our proposed BLIP-2 inversion posts new requirements for the ideal augmentations.  
% Not all augmentation methods in contrastive learning yield positive results for our purpose. 
We have conducted extensive ablation experiments to identify suitable augmentation techniques from conservative learning. 
Eventually, we employ \textit{random affine, random perspective, color-jitter, random erasing}, and \textit{Gaussian noise} and discard \textit{horizontal flipping} and \textit{random cropping}, due to a notable adverse impact on the final image 
(detailed results in \textbf{Appendix} \ref{appdx:aug_explore}). 
This is to be expected since, in our case, the considered semantic information is much more intricate than that in contrastive learning.  
For instance, some captions may contain ``left" or ``right" location information, which will be significantly altered by horizontal flipping. 
The ineffectiveness of these augmentations is also reflected in the BLIP-2 loss. 
Within the same batch of augmented images, {horizontal flipping} and {random cropping} result in larger loss values. 
Such issues related to data augmentation altering semantics have been widely studied in the context of contrastive learning \citep{chen2020simclr, tian2020goodviewforcl, kalantidis2020hardnegforcl, patrick2021comptransincl, chuang2022robustcl}.

In conclusion, with the aforementioned formulation and augmentation regularization, VLMs can be effectively inversed to generate images with both good visual quality and high alignment to given prompts. 
\textbf{Table} \ref{tab:overallresult} (``BLIP-2 INV") demonstrates SOTA alignment score on \textit{Attribute-binding} dataset \citep{huang2023t2i}.
However, since the pre-training of BLIP-2 is primarily a discriminative task, it is inevitable that a significant amount of image information is lost during its forward pass, making it challenging to recover this information through model inversion.
Additional help is required to achieve a high aesthetic quality.

\subsection{SDS for Improved Fidelity} 
% Model inversion enforces the condition while SDS fills in the details.
To make our generated images more realistic, integrating the gradient provided by another model focusing on image fidelity can help. A natural choice is Score Distillation Sampling (SDS), an optimization method based on knowledge distillation \citep{poole2022dreamfusion} that has shown great ability in generating or editing images.
We first investigate the effectiveness of SDS as a standalone image generator and measure its performance in image-text matching. Additionally, we thoroughly explore how SDS can aid model inversion in generating controllable and plausible images.

\subsection{Delicate Balance}
Our method consists of two modules, with VLM taking the lead in generating condition-compliant images and SDS ensuring the fidelity and aesthetics of the images. 
Integrating them is not a straightforward task.
In our experiments, we observe that these two components tend to prioritize different aspects and may result in divergent optimization directions. BLIP-2 inversion is primarily concerned with semantic alignment, whereas SDS prioritizes image fidelity. 
SDS often misinterprets prompts, resulting in images with missing objects or misaligned attributes when dealing with compositional prompts. 
In contrast, BLIP-2 inversion is far superior in following the instructions.
The image evolution processes of two modules individually clearly illustrate these phenomena (see \textbf{Figure} \ref{fig:evo1} and \ref{fig:evo2} in \textbf{Appendix} \ref{appdix:blipSDSrole}).
Consequently, when BLIP-2 inversion and SDS collaborate, the gradients provided by these two modules reflect different objectives (examples shown in \textbf{Appendix} \ref{appdix:blipSDSrole} \textbf{Figure} \ref{fig:grad}).
This issue calls for carefully tuning the individual weights, which we discuss and demonstrate with results in the experiment section.

\begin{algorithm}[H]
\caption{Conditional Generation via VLM Inverion and SDS}
\label{alg:our_method}
\begin{algorithmic}
    \STATE {\bfseries Input:} A text prompt $\yb$, 
    \STATE {\bfseries Required:}
   a pre-trained VLM model (e.g., BLIP2) and the loss function $L$, a pre-trained Stable Diffusion model $\phi$, and a decoder $\mathrm{DEC}$. 
    \STATE {\bfseries Output:} Generated image $\xb$ following $\yb$.
   % \REPEAT
    \STATE Set augmentation function $A$, learning rate $\beta$, weight $w_1$ for BLIP-2 and weight $w_2$ for SDS; EMA decay rate $\lambda$, EMA start iteration $e_{s}$ and EMA restart iteration $e_{rs}$.
    \STATE Initialize $\mathbf{z}_0 \gets$ sample from $\mathcal{N}(\mathbf{0}, \mathbf{I})$; $\zb_{0}^{ema}=\mathbf{z}_0$
   
    \FOR{$i=1$ {\bfseries to} $I$}
    \STATE $\mathbf{x}_i = \mathrm{DEC}(\mathbf{z}_{i-1})$
    \STATE $\nabla_{\mathbf{z}}^{'} \mathcal{L}_{SDS}(\phi, \mathbf{z}_{i-1}, \yb) = w_2 \nabla_{\mathbf{z}} \mathcal{L}_{SDS}(\phi, \mathbf{z}_{i-1}, \yb)$
    \STATE $\nabla_{\mathbf{z}}^{'} \tilde{L}_{\yb}(\mathbf{x}_i)$ = $w_1 \nabla_{\mathbf{z}} \tilde{L}_{\yb}(\mathbf{x}_i)$ $\cdot$ $\frac{||\nabla_{\mathbf{z}} \mathcal{L}_{SDS}(\phi, \mathbf{z}_{i-1}, \yb)||}{||\nabla_{\mathbf{z}} \tilde{L}_{\yb}(\mathbf{x}_i)||}$
    \STATE $\mathbf{z}_i$ $\gets$ Adam($\mathbf{z}_{i-1}$, $\beta$, $ \nabla_{\mathbf{z}}^{'} \tilde{L}_{\yb}(\mathbf{x}_i)$)
    \STATE $\mathbf{z}_i$ $\gets$ $\mathbf{z}_i - \beta$ $\nabla_{\mathbf{z}}^{'} \mathcal{L}_{SDS}(\phi, \mathbf{z}_{i-1}, \yb)$
    \STATE $\zb_{i}^{ema} \gets \lambda$ $\zb_{i-1}^{ema}$ + $(1-\lambda)$ $\zb_{i}^{ema}$
    \IF{$i\geq e_s$ and $(i-e_s) \% e_{rs} = 0 $}
    \STATE $\zb_{i} \gets \zb_{i}^{ema}$
    \ENDIF
    \ENDFOR
\STATE $\xb=\mathrm{DEC}(\mathbf{z})$
\STATE return $\xb$
   % \UNTIL{$noChange$ is $true$}
\end{algorithmic}
\end{algorithm}

\textit{Exponential moving average (EMA) restart} strategy is further proposed to enhance the stability of the two modules' collaboration. 
EMA has been widely used in deep network optimization that includes two modules or branches \citep{tarvainen2017mean, he2020moco, Cai_2021_CVPR_emabn}, showing the strength to provide stable optimization and improved prediction.
Unlike the conventional EMA method, \textit{EMA-restart} involves replacing the original optimization target with its EMA version at a specified iteration location. 
Two extra hyperparameters are involved in \textit{EMA-restart}: starting with the \textit{start} iteration location, the parameters are replaced by their EMA versions, and then every other time after the \textit{restart} interval. 
Appropriate use of \textit{EMA-restart} can effectively enhance stability during the optimization process and improve the image-text matching degree.

The overall algorithm is depicted in \textbf{Algorithm} \ref{alg:our_method}.

\section{Experiments}
Implementation details can be found in \textbf{Appendix} \ref{app:exp_detail}. We mainly represent the experiment results, analysis, and ablation studies in this section.

\begin{table*}[!ht]
\centering
\centering
\caption{Image-text alignment evaluation on T2I-CompBench. The best results are marked in green, and the second best in blue. The baseline data are sourced from \cite{chen2023pixart}.}
\label{tab:overallresult}
\begin{sc}
\begin{small}
\begin{tabular}{lcccccc}
\toprule
\multicolumn{1}{c}{\multirow{2}{*}{\textbf{Model}}}   & \multicolumn{3}{c}{\textbf{Attribute Binding}}                                                     & \multicolumn{2}{c}{\textbf{Object Relationship}}                                   & \multirow{2}{*}{\textbf{Complex}}                                            \\ \cmidrule(r){2-4} \cmidrule(r){5-6}
\multicolumn{1}{c}{}                         & \multicolumn{1}{c}{\textbf{Color} $\uparrow$}           & \multicolumn{1}{c}{\textbf{Shape} $\uparrow$}                & \multicolumn{1}{c}{\textbf{Texture} $\uparrow$}          & \multicolumn{1}{c}{\textbf{Spatial} $\uparrow$}          & \multicolumn{1}{c}{\textbf{Non-spatial} $\uparrow$}                   \\ \midrule
Stable v1.4          & 0.3765 & 0.3576 & 0.4156 & 0.1246 & 0.3079 & 0.3080 \\
% Stable v2            & 0.5065 & 0.4221 & 0.4922 & 0.1342 & 0.3096 & 0.3386 \\

Composable v2 & 0.4063 & 0.3299 & 0.3644 & 0.0800 & 0.2980 & 0.2898 \\
Structured v2 & 0.4990 & 0.4218 & 0.4900 & 0.1386 & 0.3111 & 0.3355 \\
Attn-Exct v2 & 0.6400 & 0.4517 & 0.5963 & 0.1455 & 0.3199 & 0.3401 \\
GORS & 0.6603 & 0.4785 & 0.6287 & 0.1815 & 0.3193 & 0.3328 \\
DALLE-2 & 0.5750 & 0.5464 & 0.6374 & 0.1283 & 0.3043 & 0.3696 \\
SDXL & 0.6369 & 0.5408 & 0.5637 &\cellcolor{blue!25} 0.2032 & 0.3110 & \cellcolor{blue!25}0.4091 \\
PixART-$\alpha$ & 0.6886 & 0.5582 & 0.7044 &\cellcolor{green!25} 0.2082 & 0.3179 & \cellcolor{green!25}0.4117\\ 
DALLE-3 &0.8110 & \cellcolor{green!25}0.6750 & \cellcolor{green!25}0.8070 & - & - &- \\
\midrule
SDS  v1.5 & 0.3793 & 0.3914 & 0.4321 & 0.1261 & 0.3038 & 0.2773 \\
BLIP-2 Inv & \cellcolor{green!25}0.8639 &\cellcolor{blue!25} 0.6686 &\cellcolor{blue!25} 0.7311 & 0.1008  & \cellcolor{green!25}0.3260 & 0.3379 \\
\midrule
our method  &\cellcolor{blue!25} 0.8162 & 0.6209 & 0.7202 & 0.1506 &\cellcolor{blue!25} 0.3215 & 0.3529 \\
\bottomrule
\end{tabular}
\end{small}
\end{sc}
\end{table*}

\begin{figure*}[!ht]
\begin{center}
    \rotatebox{90}{\scriptsize{~~~~~~~~SD v1-5}}
    \subfigure{
        \begin{minipage}[t]{0.11\linewidth}
            \centering
            \includegraphics[width=1\linewidth]{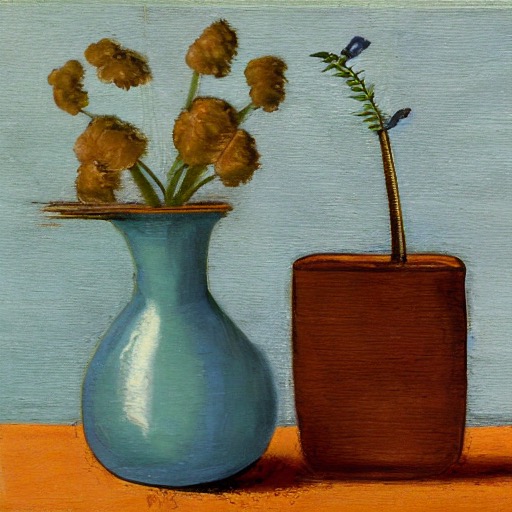}
        \end{minipage}
        \begin{minipage}[t]{0.11\linewidth}
            \centering
            \includegraphics[width=1\linewidth]{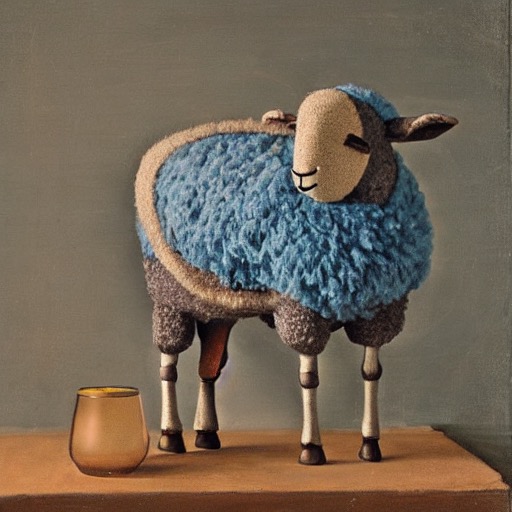}
        \end{minipage}
    }
    \subfigure{
        \begin{minipage}[t]{0.11\linewidth}
            \centering
            \includegraphics[width=1\linewidth]{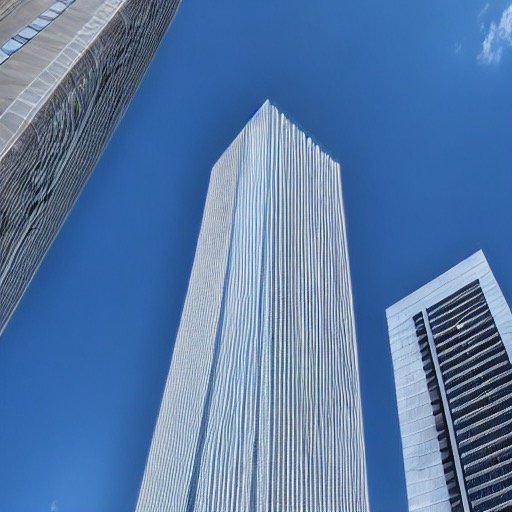}
        \end{minipage}
        \begin{minipage}[t]{0.11\linewidth}
            \centering
            \includegraphics[width=1\linewidth]{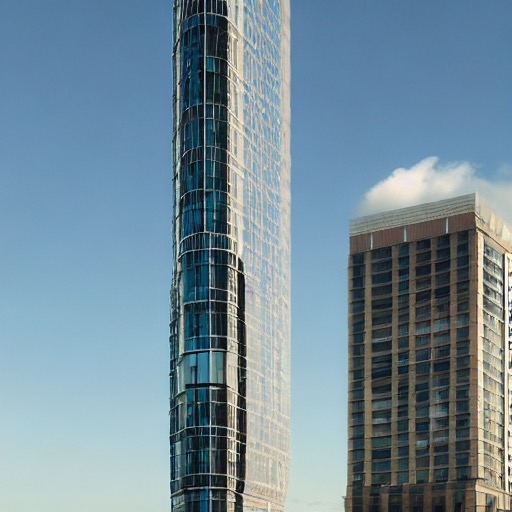}
        \end{minipage}
    }
    \subfigure{
        \begin{minipage}[t]{0.11\linewidth}
            \centering
            \includegraphics[width=1\linewidth]{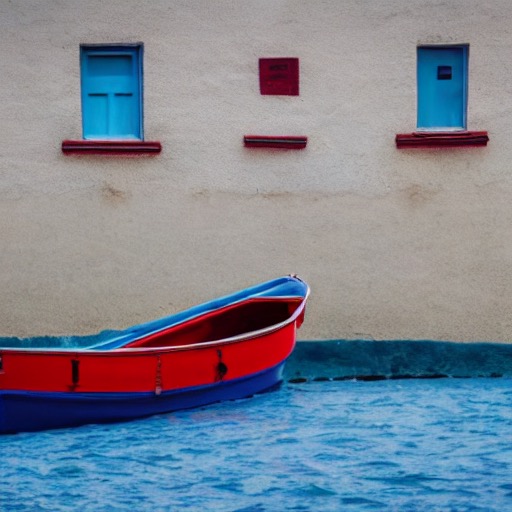}
            \end{minipage}
        \begin{minipage}[t]{0.11\linewidth}
            \centering
            \includegraphics[width=1\linewidth]{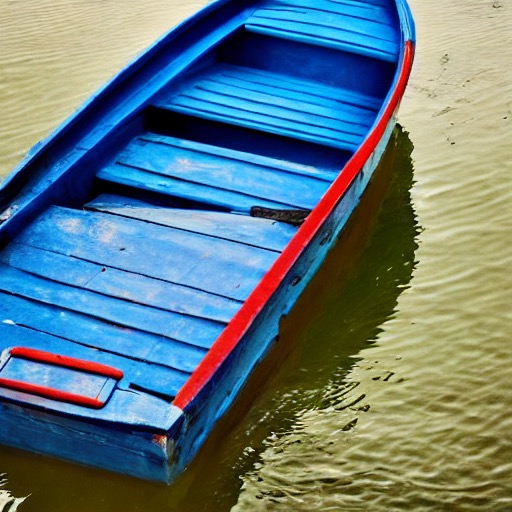}
        \end{minipage}
    }
    \subfigure{
        \begin{minipage}[t]{0.11\linewidth}
            \centering
            \includegraphics[width=1\linewidth]{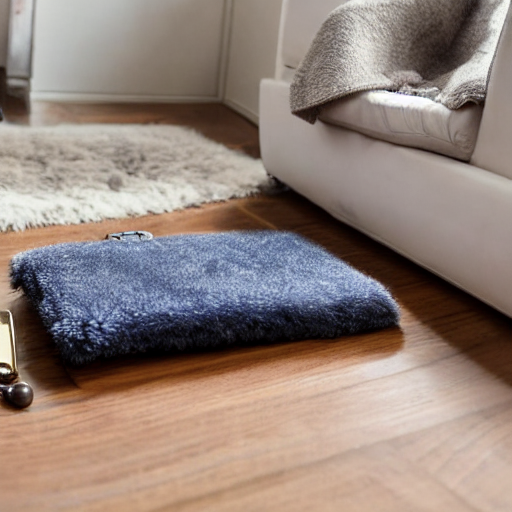}
            \end{minipage}
        \begin{minipage}[t]{0.11\linewidth}
            \centering
            \includegraphics[width=1\linewidth]{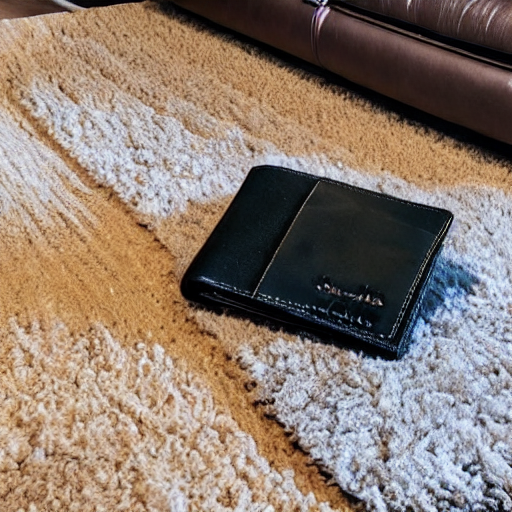}
        \end{minipage}
    }
\end{center}

\begin{center}
    \rotatebox{90}{\scriptsize{~~~~~~~Struc-Diff}}
    \subfigure{
        \begin{minipage}[t]{0.11\linewidth}
            \centering
            \includegraphics[width=1\linewidth]{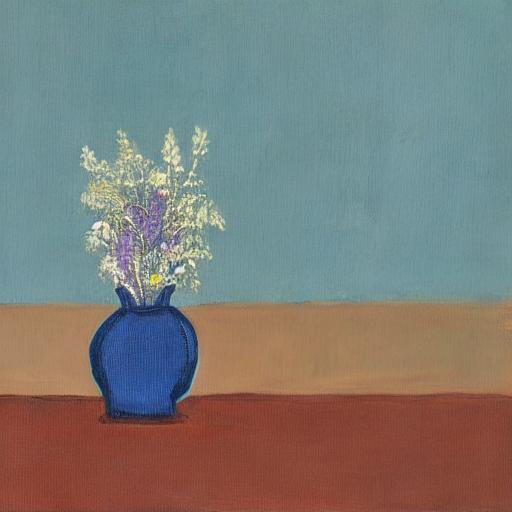}
        \end{minipage}
        \begin{minipage}[t]{0.11\linewidth}
            \centering
            \includegraphics[width=1\linewidth]{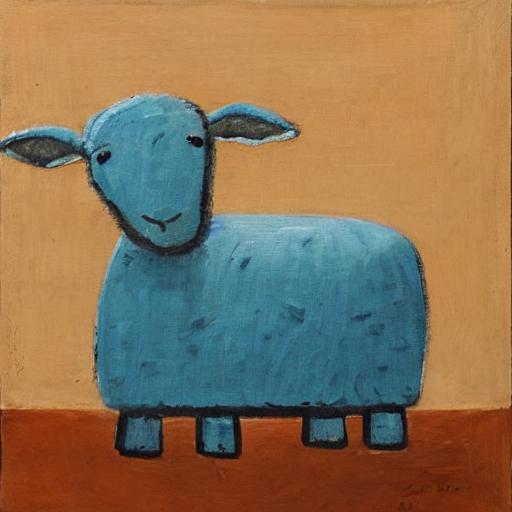}
        \end{minipage}
    }
    \subfigure{
        \begin{minipage}[t]{0.11\linewidth}
            \centering
            \includegraphics[width=1\linewidth]{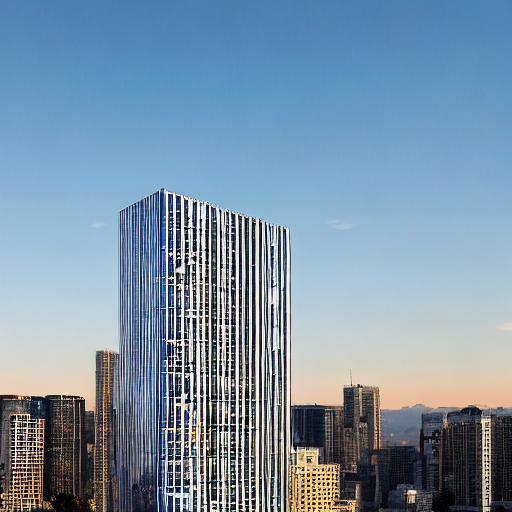}
        \end{minipage}
        \begin{minipage}[t]{0.11\linewidth}
            \centering
            \includegraphics[width=1\linewidth]{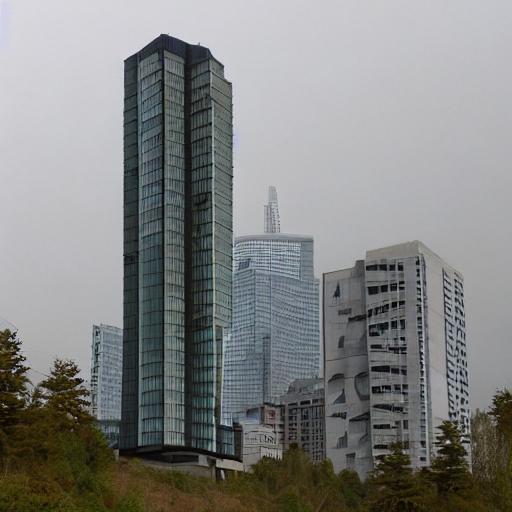}
        \end{minipage}
    }
    \subfigure{
        \begin{minipage}[t]{0.11\linewidth}
            \centering
            \includegraphics[width=1\linewidth]{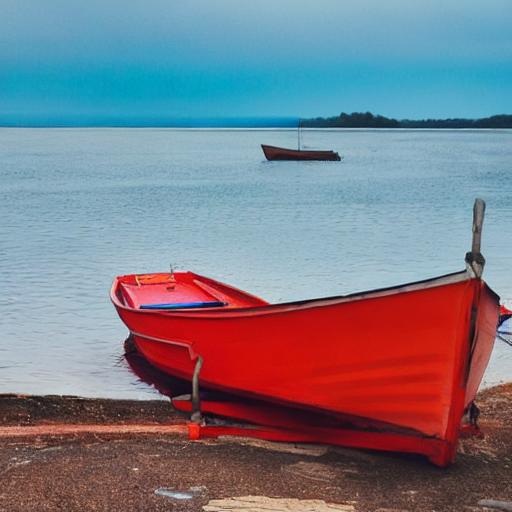}
            \end{minipage}
        \begin{minipage}[t]{0.11\linewidth}
            \centering
            \includegraphics[width=1\linewidth]{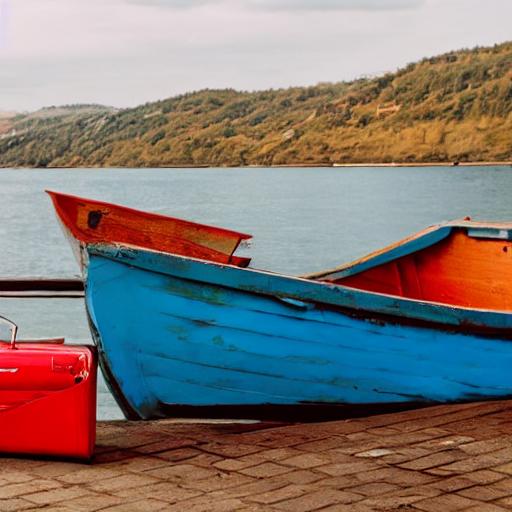}
        \end{minipage}
    }
    \subfigure{
        \begin{minipage}[t]{0.11\linewidth}
            \centering
            \includegraphics[width=1\linewidth]{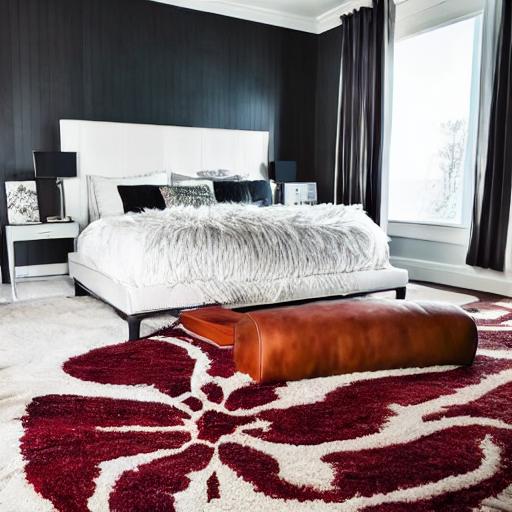}
            \end{minipage}
        \begin{minipage}[t]{0.11\linewidth}
            \centering
            \includegraphics[width=1\linewidth]{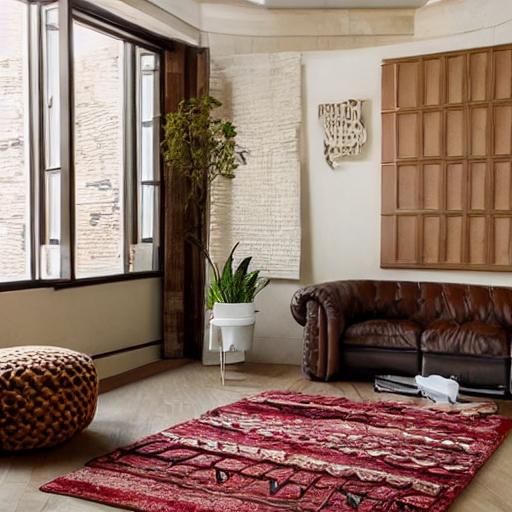}
        \end{minipage}
    }
\end{center}

\begin{center}
    \rotatebox{90}{\scriptsize{~~~~~~~~Attn-Exct}}
    \subfigure{
        \begin{minipage}[t]{0.11\linewidth}
            \centering
            \includegraphics[width=1\linewidth]{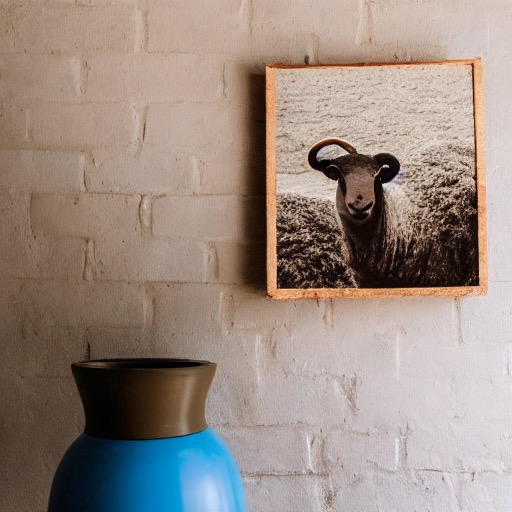}
        \end{minipage}
        \begin{minipage}[t]{0.11\linewidth}
            \centering
            \includegraphics[width=1\linewidth]{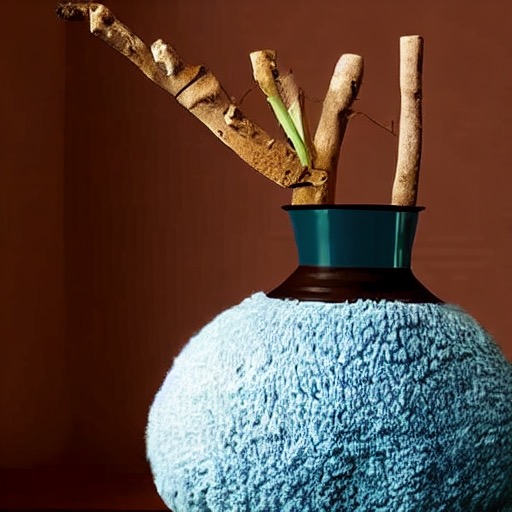}
        \end{minipage}
    }
    \subfigure{
        \begin{minipage}[t]{0.11\linewidth}
            \centering
            \includegraphics[width=1\linewidth]{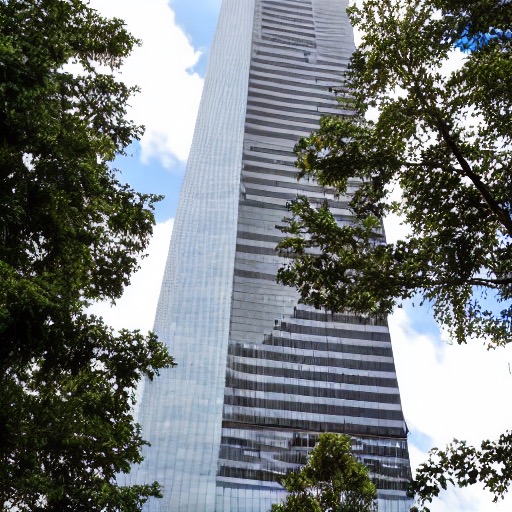}
        \end{minipage}
        \begin{minipage}[t]{0.11\linewidth}
            \centering
            \includegraphics[width=1\linewidth]{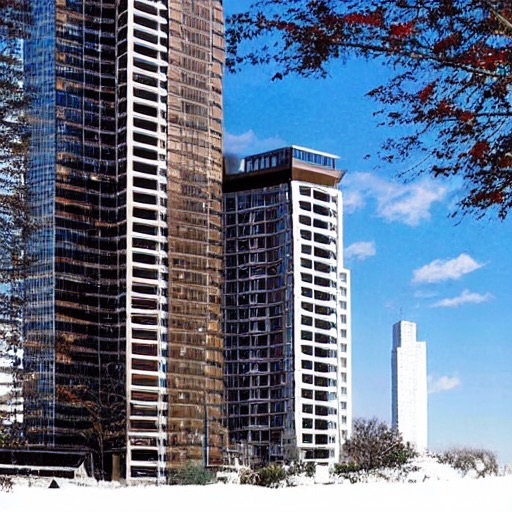}
        \end{minipage}
    }
    \subfigure{
        \begin{minipage}[t]{0.11\linewidth}
            \centering
            \includegraphics[width=1\linewidth]{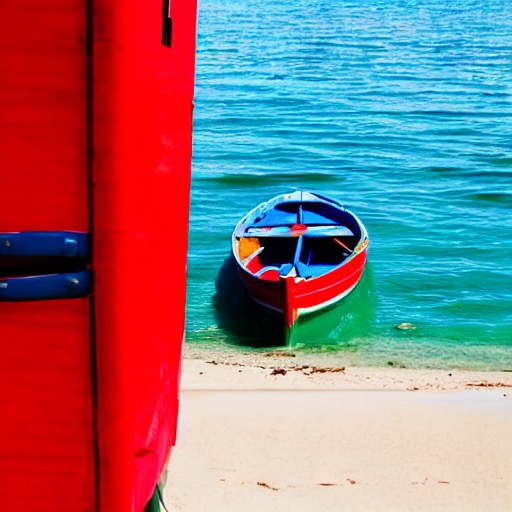}
            \end{minipage}
        \begin{minipage}[t]{0.11\linewidth}
            \centering
            \includegraphics[width=1\linewidth]{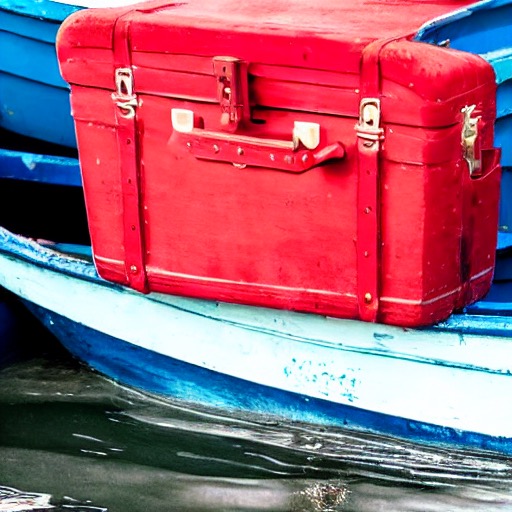}
        \end{minipage}
    }
    \subfigure{
        \begin{minipage}[t]{0.11\linewidth}
            \centering
            \includegraphics[width=1\linewidth]{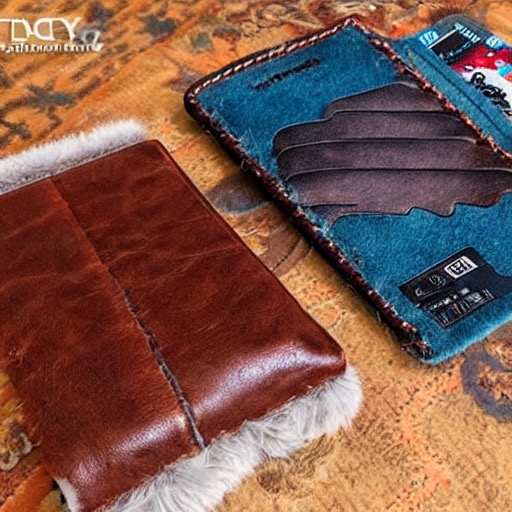}
            \end{minipage}
        \begin{minipage}[t]{0.11\linewidth}
            \centering
            \includegraphics[width=1\linewidth]{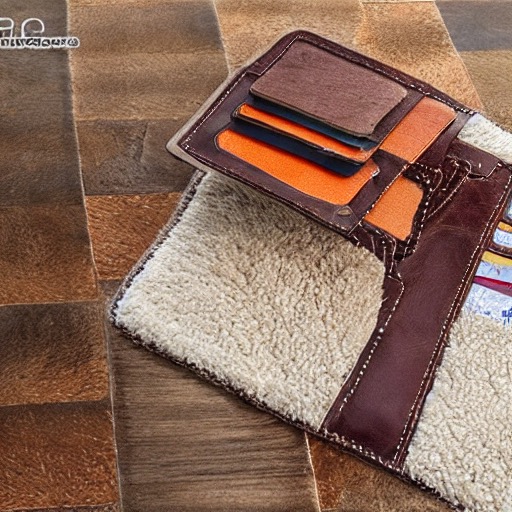}
        \end{minipage}
    }
\end{center}

\setcounter{subfigure}{0}
\begin{center}
    \rotatebox{90}{\scriptsize{~~~~~~Our method}}
    \subfigure[A blue sheep and a brown vase]{
        \begin{minipage}[t]{0.11\linewidth}
            \centering
            \includegraphics[width=1\linewidth]{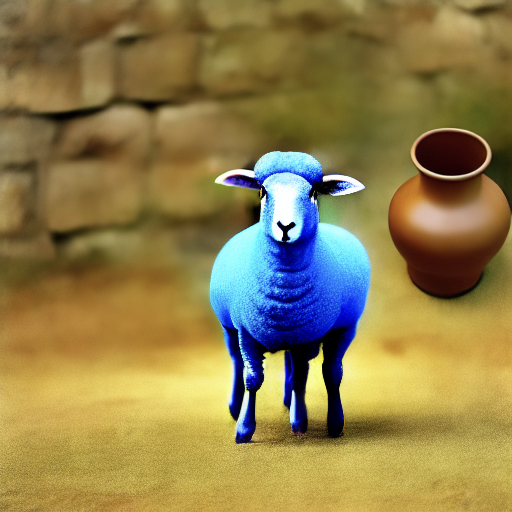}
        \end{minipage}
        \begin{minipage}[t]{0.11\linewidth}
            \centering
            \includegraphics[width=1\linewidth]{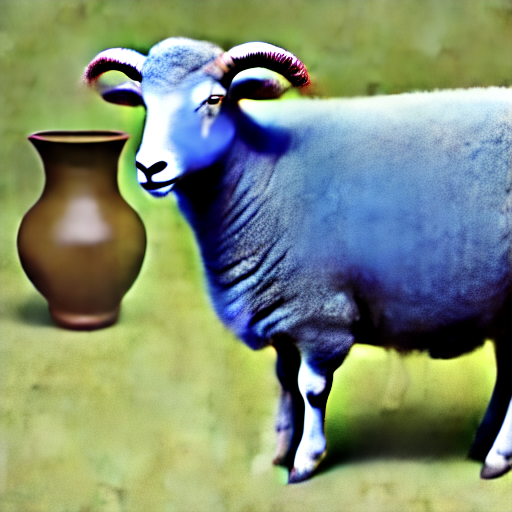}
        \end{minipage}
    }
    \subfigure[A high skyscraper and a small cabin]{
        \begin{minipage}[t]{0.11\linewidth}
            \centering
            \includegraphics[width=1\linewidth]{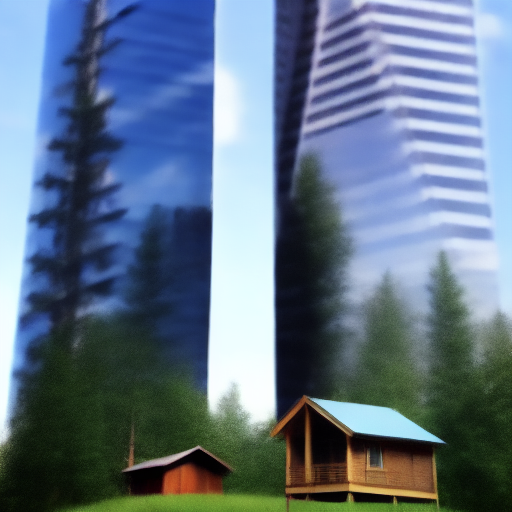}
        \end{minipage}
        \begin{minipage}[t]{0.11\linewidth}
            \centering
            \includegraphics[width=1\linewidth]{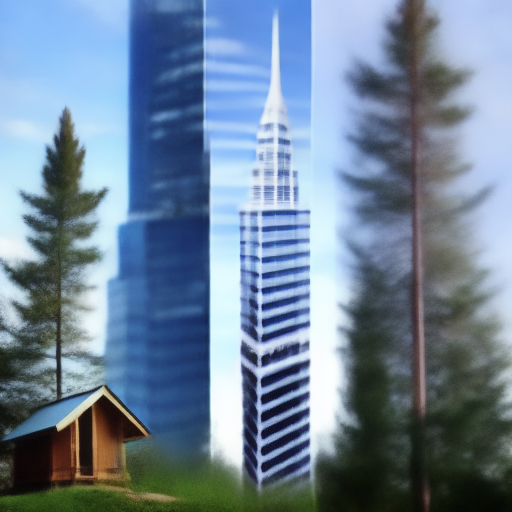}
        \end{minipage}
        }
    \subfigure[A blue boat and a red suitcase]{
        \begin{minipage}[t]{0.11\linewidth}
            \centering
            \includegraphics[width=1\linewidth]{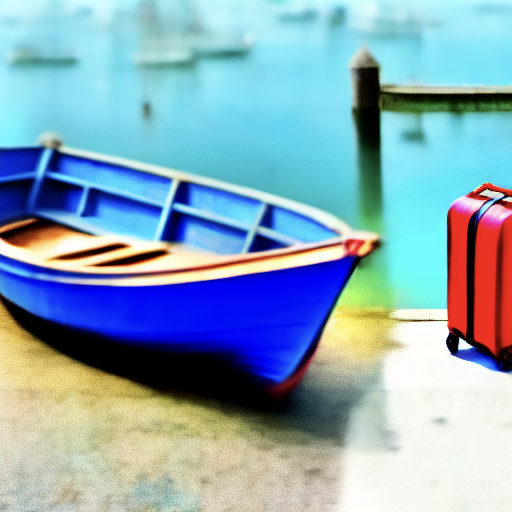}
        \end{minipage}
            \begin{minipage}[t]{0.11\linewidth}
            \centering
            \includegraphics[width=1\linewidth]{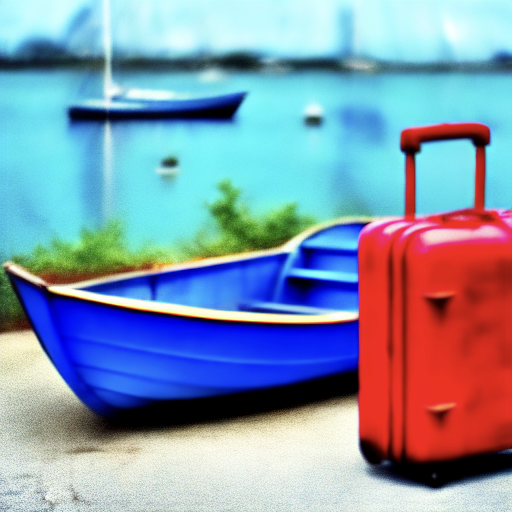}
        \end{minipage}
        }
    \subfigure[A fluffy rug and a leather wallet]{
        \begin{minipage}[t]{0.11\linewidth}
            \centering
            \includegraphics[width=1\linewidth]{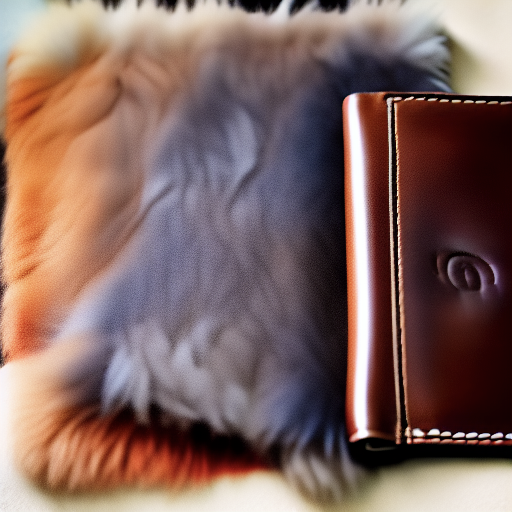}
        \end{minipage}
            \begin{minipage}[t]{0.11\linewidth}
            \centering
            \includegraphics[width=1\linewidth]{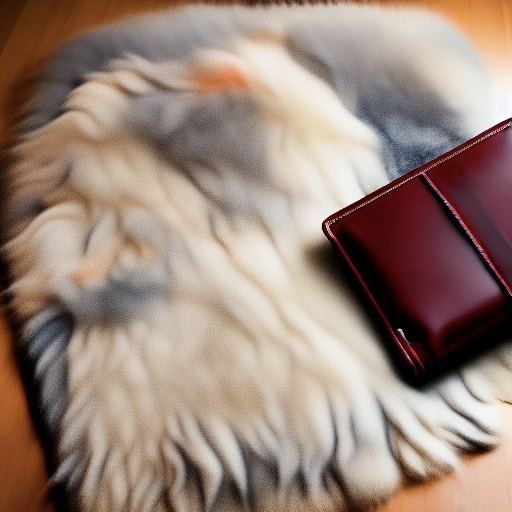}
        \end{minipage}
        }
\end{center}
\caption{Qualitative Comparison using prompts from \textit{Attribute-binding} of T2I-CompBench \citep{huang2023t2i}. We generate two images for each prompt with the same two random seeds for all methods.}
\label{fig:quali}
\end{figure*}

\begin{table*}[ht]
\caption{Results with different EMA start iterations $e_s$ and restart intervals $e_{rs}$. Baseline results include without EMA: 67.43 (75.76/48.36/78.71); with original EMA: 67.22 (76.98/47.36/77.32).}
\label{tab:ema}
\centering
\begin{sc}
\begin{small}
\begin{tabular}{c|cccc}
\toprule
\multicolumn{1}{l}{$e_{rs}$\textbackslash{} $e_s$} & 20 & 40 & 60 & 80\\ \midrule
10                                               & \begin{tabular}[c]{@{}c@{}}65.49\\ (76.86/46.25/73.35)\end{tabular} & \begin{tabular}[c]{@{}c@{}}67.75\\ (77.24/49.52/76.50)\end{tabular} & \begin{tabular}[c]{@{}c@{}}65.89\\ (75.18/47.06/75.41)\end{tabular} & \begin{tabular}[c]{@{}c@{}}67.15\\ (78.10/45.19/78.16)\end{tabular} \\ \midrule
20                                               & \begin{tabular}[c]{@{}c@{}}67.96\\ (81.70/49.91/72.27)\end{tabular} & \begin{tabular}[c]{@{}c@{}}65.76\\ (77.06/43.04/72.18)\end{tabular} & \begin{tabular}[c]{@{}c@{}}62.61\\ (73.00/46.68/78.14)\end{tabular} & \begin{tabular}[c]{@{}c@{}}\textbf{71.77}\\ (\textbf{79.47/50.56/85.29})\end{tabular} \\ \midrule
40                                               & \begin{tabular}[c]{@{}c@{}}64.55\\ (72.53/48.51/77.61)\end{tabular} & \begin{tabular}[c]{@{}c@{}}65.11\\ (69.89/48.77/76.68)\end{tabular} & \begin{tabular}[c]{@{}c@{}}\textbf{71.86}\\ (\textbf{80.36/49.06/86.16})\end{tabular} & \begin{tabular}[c]{@{}c@{}}66.73\\ (76.20/38.20/75.99)\end{tabular} \\ \midrule
60                                               & \begin{tabular}[c]{@{}c@{}}\textbf{71.24}\\ (\textbf{68.90/50.43/84.40})\end{tabular} & \begin{tabular}[c]{@{}c@{}}\textbf{72.47}\\ (\textbf{83.99/53.41/80.01})\end{tabular} & \begin{tabular}[c]{@{}c@{}}67.14\\ (78.52/46.39/76.50)\end{tabular} & \begin{tabular}[c]{@{}c@{}}66.73\\ (76.20/38.20/75.99)\end{tabular} \\ \bottomrule
\end{tabular}
\end{small}
\end{sc}
\end{table*}

\subsection{Quantitative Results}
We evaluate our method on T2I-CompBench \citep{huang2023t2i}, a comprehensive benchmark for open-world compositional text-to-image generation consisting of attribute binding, object relationships, and complex compositions. The results are presented in \textbf{Table} \ref{tab:overallresult}. 

\subsection{Qualitative Comparisons}
We compare our method with Stable Difussion v1-5, StructureDiffusion \citep{feng2022strucdiffusion} and Attend-and-Excite \citep{phung2023att-refocus}.
The qualitative results are present in \textbf{Figure} \ref{fig:quali}. Clearly, baseline methods all tend to neglect objects in the prompt.
Additionally, we provide more examples of images generated by our method in \textbf{Appendix} \ref{appdix:moreimages}, including prompts from \textit{Object Relationship} and \textit{Complex} sub-datasets from T2I-CompBench. As can be seen, our method is also able to generate faithful images for prompts that describe actions, like ``A person is walking with a friend and catching up on their lives", and complex prompts that include multiple attributes for each object, like ``The bold, striking patterns of the tiger's stripes blended seamlessly with the dappled light of the jungle, a creature of stealth and beauty" (see \textbf{Figure} \ref{fig:cherry_com}).

\subsection{Ablation Studies}
Our method comprises two pivotal components: a pre-trained BLIP-2 model and SDS with SD. Among these, BLIP-2 assumes a more critical role in the image generation process, ensuring that the generated images faithfully adhere to the provided texture instructions. To harmoniously integrate these two components, extensive exploratory experiments were conducted to identify the optimal combination. In our ablation studies, we randomly sample 60 images, with 20 images selected from each of the three \textit{Attribute-binding} datasets (color/shape/texture) within T2I-CompBench for each experiment.

\paragraph{SDS Weight}
We adaptively scale the BLIP-2 gradient so that its norm is twice that of the SDS gradient and conduct the exploration to determine the most suitable weights for SDS. 
We observe that higher values of the CFG factor and SDS weight $w_2$ can produce clearer and more realistic images but at the cost of reduced image-text alignment. Conversely, extremely low SDS weights, such as those starting below 800, struggle to generate clear and natural-looking images. Similarly, when CFG values are low, for instance, at 10, the image quality and alignment are both harmed. The most suitable SDS weight falls within the range of 800 to 1000, ideally paired with CFG values of 20 or 30. The quantitative result is presented in \textbf{Appendix} \ref{appdix:moreablation}.
We further explored the scaled SDS weight and found a sweet point at starting with 800 and gradually decreasing to 400, with CFG equal to 30 (see \textbf{Table} \ref{tab:sds_weight_scaling}).

\paragraph{BLIP-2 Weight}
Since we optimize the image with the gradient of BLIP-2 and SDS separately, we also investigate the frequency and weight of BLIP-2 $w_1$ at each iteration. The results are shown in \textbf{Table} \ref{tab:blip_weight}. The BLIP gradient norm should always be larger than SDS to ensure a better alignment with the text, yet an overemphasis on it can lead to the generation of noisy and unrealistic images, which ultimately reduces the alignment score.

\begin{table}[ht]
\caption{Results of scaling SDS weight $w_2$ with CFG scale of 20 \& 30}
\label{tab:sds_weight_scaling}
\centering
\scalebox{0.75}{
\begin{tabular}{c|ccc}
\toprule
{cfg \textbackslash{}$ w_2$} &{800 $\to$ 400} & { 1000 $\to$ 500} & { 1500 $\to$ 500} \\ \midrule
20 & \begin{tabular}[c]{@{}c@{}}63.77\\ (71.32/44.78/75.21)\end{tabular} & \begin{tabular}[c]{@{}c@{}}64.12\\ (74.36/48.49/70.53)\end{tabular} & \begin{tabular}[c]{@{}c@{}}60.67\\ (71.23/42.89/68.12)\end{tabular} \\ \midrule
30 & \begin{tabular}[c]{@{}c@{}}\textbf{69.43}\\ (\textbf{79.98/49.08/79.23})\end{tabular} & \begin{tabular}[c]{@{}c@{}}\textbf{67.43}\\ (\textbf{75.76/48.36/78.17})\end{tabular} & \begin{tabular}[c]{@{}c@{}}63.34\\ (74.00/43.19/72.89)\end{tabular} \\ 
\bottomrule
\end{tabular}
}
% \end{sc}
\end{table}

\begin{table}[ht]
\caption{The influence of BLIP-2 weight $w_1$ and frequency.}
\label{tab:blip_weight}
\scalebox{0.76}{
\begin{tabular}{c|ccc}
\hline
{f \textbackslash{} $w_1$} & 1 & 2 & 3\\ \hline
1 & \begin{tabular}[c]{@{}c@{}}59.97\\ (75.08/39.44/65.37)\end{tabular} & \begin{tabular}[c]{@{}c@{}}\textbf{67.43}\\ (\textbf{75.76/48.36/78.17})\end{tabular} & \begin{tabular}[c]{@{}c@{}}62.36\\ (65.12/46.91/75.04)\end{tabular} \\ \hline
2  & \begin{tabular}[c]{@{}c@{}}61.55\\ (72.96/46.66/65.02)\end{tabular} & \begin{tabular}[c]{@{}c@{}}64.29\\ (75.52/43.77/73.59)\end{tabular} & -   \\ \hline
\end{tabular}
}
\end{table}

\paragraph{EMA-restart}
We explore the effectiveness of the EMA-restart strategy with different \textit{start} iteration locations $e_s$ and \textit{restart} intervals $e_{rs}$, and the results are presented in \textbf{Table} \ref{tab:ema}.
In general, we find that combinations of \textit{start} $\&$ \textit{restart} that resulted in 2$\sim$3 times EMA replacements can significantly benefit the result.

\textbf{Random Noise}
Our experiments indicate that introducing random noise to $\zb$ after each optimization iteration enhances result robustness. The results of different random noise scales are displayed in \textbf{Appendix} \ref{appdix:moreablation} \textbf{Table} \ref{tab_noise}.

\section{Related Work}
\paragraph{Controllable Image Generation.}

To facilitate flexible image generation capable of meeting diverse requirements, achieving controllable image generation has become a prominent focus in recent times. A prevalent approach is guidance generation. 
This kind of method obviates the necessity of training a diffusion model from scratch with specific conditions and instead offers direction by pre-existing specialized models or loss functions for the reverse process of the diffusion model.
An early attempt in this direction was GLIDE \citep{nichol2021glide} and classifier-free guidance \citep{ho2022classifier}, which facilitate a diffusion model in generating images that align with textual descriptions.
To satisfy arbitrary requirements like textual descriptions, layouts, or segmentation, Universal Guidance \citep{bansal2023universal} is proposed. It leverages pre-existing models to offer iterative directions during the reverse process of a diffusion model.
Specifically, to achieve control over multi-faced objects within the prompts, current methods typically involve the incorporation of prior layout information, either by utilizing the bound boxes as guidance \citep{lian2023llmground} or injecting grounding information into new trainable layers \citep{li2023gligen,chen2023reasonlayut}. Furthermore, it has been observed that cross-attention primarily governs the handling of object-related information. Consequently, various approaches have been developed that involve modifying cross-attention mechanisms to ensure that the diffusion model sufficiently attends to all objects specified in the prompts \cite{feng2022strucdiffusion, kim2023densediff,chefer2023attendandexcite}. Notably, approaches that combine bounding boxes with attention control have demonstrated improved performance in this regard \citep{phung2023att-refocus, wang2023compositional-attcontrol}.

What sets our method apart is the central role played by a discriminative Vision-Language Model (VLM) during the image synthesis process, rather than relying on the generative model. Additionally, our approach is entirely training-free and does not necessitate the inclusion of extra information, such as layout details or object indices.

\paragraph{Image Generation via Model Inversion}
Model inversion refers to the invert process of general model training, i.e., optimizing the input, which is initially randomly initialized, while keeping the well-trained model parameters unchanged \citep{mahendran2015imageinvert}. An early endeavor in this direction was DeepDream \citep{deepdream}, which sought to create images corresponding to specific classes given a classification model, by producing high responses for specific classes in the output layer of the model. However, the reverse classification process posed significant challenges. 
During the optimization process of model inversion, it's easy to get stuck in local optima and end up with very unrealistic images that can be regarded as adversarial examples \citep{goodfellow2014explaining, wang2022traditionalCaG}.
Subsequent methods, such as DeepInversion \citep{yin2020deepinversion} and CaG \citep{wang2022traditionalCaG}, introduced various regularization techniques to improve the identifiability of generated images. However, the generated images were still far from natural and lacked coherency. Notably, VQGAN-CLIP \citep{crowson2022vqganclip} emerged as a successful example within this category of methods. It generates images by inversely applying a trained VQGAN model \citep{esser2021taming} using a loss function that measures the similarity of image and text embeddings from a CLIP model \citep{radford2021clip} and incorporates augmentation regularization. Nevertheless, VQGAN-CLIP did not delve into the impact of various augmentation methods and the mechanism underlying the effectiveness of augmentation regularization, even in subsequent studies.

We are the first to comprehensively investigate the underlying mechanism of augmentation regularization, and how to choose the appropriate augmentation techniques.

\section{Conclusion and Discussion}

% Illustrate its implication for future work.
In this study, we introduce a novel framework for image generation that enhances controllability. Our approach is rooted in a fresh perspective on comprehending the recent advancements in text-to-image synthesis models like DALLE3 as a learned inversion function of the VLM.
Subsequently, we propose a direct inversion of the VLM through optimization, harnessing the full potential of text and image alignment inherent in VLMs. We elucidate the effectiveness of \textit{augmentation regularization} that facilitates generating faithful images via VLM inversion.
Furthermore, we enhance our method by incorporating SDS and thoroughly explore the effective synergy between it and VLM, achieving correctly generating fidelity images.

Nonetheless, our method does have limitations. 
For instance, the stability of the generation process is not as good as SOTA DPMs, due to the optimization nature.  Occasionally mismatches or unrealistic images still occur. 
In this work, we only demonstrated working with the BLIP-2 model as a referee, which we found to struggle with spatial information. Instructions involving spatial relationships such as ``to the left/right", ``on the top/bottom", etc., can be challenging to follow. 
Our work can be further strengthened if we can incorporate multiple referees with diverse knowledge, e.g., grounding VQA model, detection models, segmentation models, etc. 
Such a model zoo setting has been demonstrated effective in domain generalization \citep{shu2021zoo,dong2022zood, chen2023explore}. As the referees get more sophisticated, memory-efficient optimization \citep{anil2019memory, baydin2022gradients, malladi2024fine} can be utilized to reduce the computation load and promote scalability.

Despite these limitations, we are introducing a novel direction and highlighting its significant potential. For all modalities, evaluating the generated samples is a challenging task. For example, we have FID \citep{heusel2017gans_fid} for images, but such a well-accepted criterion is missing for 3D generation. Nevertheless, as long as there exist powerful discriminative models that can tell whether the generated samples are good or not, they can be utilized in a training-free fashion by extending our method to achieve a significant boost in conditional generation across every modality.
Our goal is to inspire further investigations along this path, revealing the unexplored capabilities of discriminative models in generative tasks.
Currently, we have only explored basic models in image generation, and we hope to inspire more future works to leverage more powerful discriminative and language models, thereby developing more robust capabilities.

\bibliography{reference}

\begin{thebibliography}{72}
\providecommand{\natexlab}[1]{#1}
\providecommand{\url}[1]{\texttt{#1}}
\expandafter\ifx\csname urlstyle\endcsname\relax
  \providecommand{\doi}[1]{doi: #1}\else
  \providecommand{\doi}{doi: \begingroup \urlstyle{rm}\Url}\fi

\bibitem[Adams \& Fournier(2003)Adams and Fournier]{adams2003sobolev}
Adams, R.~A. and Fournier, J.~J.
\newblock \emph{Sobolev Spaces}, volume 140.
\newblock Academic Press, 2003.

\bibitem[Alexander~Mordvintsev \& Tyka(2015)Alexander~Mordvintsev and Tyka]{deepdream}
Alexander~Mordvintsev, C.~O. and Tyka, M.
\newblock Inceptionism: Going deeper into neural networks, 2015.
\newblock URL \url{https://blog.research.google/2015/06/inceptionism-going-deeper-into-neural.html}.

\bibitem[Anil et~al.(2019)Anil, Gupta, Koren, and Singer]{anil2019memory}
Anil, R., Gupta, V., Koren, T., and Singer, Y.
\newblock Memory efficient adaptive optimization.
\newblock \emph{Advances in Neural Information Processing Systems}, 32, 2019.

\bibitem[Arrieta et~al.(2020)Arrieta, D{\'\i}az-Rodr{\'\i}guez, Del~Ser, Bennetot, Tabik, Barbado, Garc{\'\i}a, Gil-L{\'o}pez, Molina, Benjamins, et~al.]{arrieta2020explainable}
Arrieta, A.~B., D{\'\i}az-Rodr{\'\i}guez, N., Del~Ser, J., Bennetot, A., Tabik, S., Barbado, A., Garc{\'\i}a, S., Gil-L{\'o}pez, S., Molina, D., Benjamins, R., et~al.
\newblock Explainable artificial intelligence (xai): Concepts, taxonomies, opportunities and challenges toward responsible ai.
\newblock \emph{Information fusion}, 58:\penalty0 82--115, 2020.

\bibitem[Bansal et~al.(2023)Bansal, Chu, Schwarzschild, Sengupta, Goldblum, Geiping, and Goldstein]{bansal2023universal}
Bansal, A., Chu, H.-M., Schwarzschild, A., Sengupta, S., Goldblum, M., Geiping, J., and Goldstein, T.
\newblock Universal guidance for diffusion models.
\newblock In \emph{Proceedings of the IEEE/CVF Conference on Computer Vision and Pattern Recognition}, pp.\  843--852, 2023.

\bibitem[Baydin et~al.(2022)Baydin, Pearlmutter, Syme, Wood, and Torr]{baydin2022gradients}
Baydin, A.~G., Pearlmutter, B.~A., Syme, D., Wood, F., and Torr, P.
\newblock Gradients without backpropagation.
\newblock \emph{arXiv preprint arXiv:2202.08587}, 2022.

\bibitem[Bertsekas(2016)]{bertsekas2016nonlinear}
Bertsekas, D.
\newblock \emph{Nonlinear Programming}.
\newblock Athena scientific optimization and computation series. Athena Scientific, 2016.

\bibitem[Betker et~al.(2023)Betker, Goh, Jing, Brooks, Wang, Li, Ouyang, Zhuang, Lee, Guo, et~al.]{betker2023dalle3}
Betker, J., Goh, G., Jing, L., Brooks, T., Wang, J., Li, L., Ouyang, L., Zhuang, J., Lee, J., Guo, Y., et~al.
\newblock Improving image generation with better captions.
\newblock \emph{Computer Science. https://cdn. openai. com/papers/dall-e-3. pdf}, 2023.

\bibitem[Cai et~al.(2021)Cai, Ravichandran, Maji, Fowlkes, Tu, and Soatto]{Cai_2021_CVPR_emabn}
Cai, Z., Ravichandran, A., Maji, S., Fowlkes, C., Tu, Z., and Soatto, S.
\newblock Exponential moving average normalization for self-supervised and semi-supervised learning.
\newblock In \emph{Proceedings of the IEEE/CVF Conference on Computer Vision and Pattern Recognition (CVPR)}, pp.\  194--203, June 2021.

\bibitem[Chefer et~al.(2023)Chefer, Alaluf, Vinker, Wolf, and Cohen-Or]{chefer2023attendandexcite}
Chefer, H., Alaluf, Y., Vinker, Y., Wolf, L., and Cohen-Or, D.
\newblock Attend-and-excite: Attention-based semantic guidance for text-to-image diffusion models.
\newblock \emph{ACM Transactions on Graphics (TOG)}, 42\penalty0 (4):\penalty0 1--10, 2023.

\bibitem[Chen et~al.(2023{\natexlab{a}})Chen, Yu, Ge, Yao, Xie, Wu, Wang, Kwok, Luo, Lu, et~al.]{chen2023pixart}
Chen, J., Yu, J., Ge, C., Yao, L., Xie, E., Wu, Y., Wang, Z., Kwok, J., Luo, P., Lu, H., et~al.
\newblock Pixart-$\alpha$: Fast training of diffusion transformer for photorealistic text-to-image synthesis.
\newblock \emph{arXiv preprint arXiv:2310.00426}, 2023{\natexlab{a}}.

\bibitem[Chen et~al.(2020)Chen, Kornblith, Norouzi, and Hinton]{chen2020simclr}
Chen, T., Kornblith, S., Norouzi, M., and Hinton, G.
\newblock A simple framework for contrastive learning of visual representations.
\newblock In \emph{International conference on machine learning}, pp.\  1597--1607. PMLR, 2020.

\bibitem[Chen et~al.(2023{\natexlab{b}})Chen, Liu, Yang, Yuan, You, Liu, and Yang]{chen2023reasonlayut}
Chen, X., Liu, Y., Yang, Y., Yuan, J., You, Q., Liu, L.-P., and Yang, H.
\newblock Reason out your layout: Evoking the layout master from large language models for text-to-image synthesis.
\newblock \emph{arXiv preprint arXiv:2311.17126}, 2023{\natexlab{b}}.

\bibitem[Chen et~al.(2023{\natexlab{c}})Chen, Hu, Zhou, Li, and Ma]{chen2023explore}
Chen, Y., Hu, T., Zhou, F., Li, Z., and Ma, Z.-M.
\newblock Explore and exploit the diverse knowledge in model zoo for domain generalization.
\newblock In \emph{International Conference on Machine Learning}, pp.\  4623--4640. PMLR, 2023{\natexlab{c}}.

\bibitem[Chuang et~al.(2022)Chuang, Hjelm, Wang, Vineet, Joshi, Torralba, Jegelka, and Song]{chuang2022robustcl}
Chuang, C.-Y., Hjelm, R.~D., Wang, X., Vineet, V., Joshi, N., Torralba, A., Jegelka, S., and Song, Y.
\newblock Robust contrastive learning against noisy views.
\newblock In \emph{Proceedings of the IEEE/CVF Conference on Computer Vision and Pattern Recognition}, pp.\  16670--16681, 2022.

\bibitem[Cohen et~al.(2019)Cohen, Rosenfeld, and Kolter]{cohen2019certified}
Cohen, J., Rosenfeld, E., and Kolter, Z.
\newblock Certified adversarial robustness via randomized smoothing.
\newblock In \emph{international conference on machine learning}, pp.\  1310--1320. PMLR, 2019.

\bibitem[Crowson et~al.(2022)Crowson, Biderman, Kornis, Stander, Hallahan, Castricato, and Raff]{crowson2022vqganclip}
Crowson, K., Biderman, S., Kornis, D., Stander, D., Hallahan, E., Castricato, L., and Raff, E.
\newblock Vqgan-clip: Open domain image generation and editing with natural language guidance.
\newblock In \emph{European Conference on Computer Vision}, pp.\  88--105. Springer, 2022.

\bibitem[Dangovski et~al.(2021)Dangovski, Jing, Loh, Han, Srivastava, Cheung, Agrawal, and Solja{\v{c}}i{\'c}]{dangovski2021equivariant}
Dangovski, R., Jing, L., Loh, C., Han, S., Srivastava, A., Cheung, B., Agrawal, P., and Solja{\v{c}}i{\'c}, M.
\newblock Equivariant contrastive learning.
\newblock \emph{arXiv preprint arXiv:2111.00899}, 2021.

\bibitem[Ding et~al.(2023)Ding, Hu, Jiang, Li, Wang, and Yao]{ding2023random}
Ding, L., Hu, T., Jiang, J., Li, D., Wang, W., and Yao, Y.
\newblock Random smoothing regularization in kernel gradient descent learning.
\newblock \emph{arXiv preprint arXiv:2305.03531}, 2023.

\bibitem[Dong et~al.(2022)Dong, Muhammad, Zhou, Xie, Hu, Yang, Bae, and Li]{dong2022zood}
Dong, Q., Muhammad, A., Zhou, F., Xie, C., Hu, T., Yang, Y., Bae, S.-H., and Li, Z.
\newblock Zood: Exploiting model zoo for out-of-distribution generalization.
\newblock \emph{Advances in Neural Information Processing Systems Volume 35}, 2022.

\bibitem[Du et~al.(2022)Du, Liu, Li, and Zhao]{du2022vlmsurvey}
Du, Y., Liu, Z., Li, J., and Zhao, W.~X.
\newblock A survey of vision-language pre-trained models.
\newblock \emph{arXiv preprint arXiv:2202.10936}, 2022.

\bibitem[Esser et~al.(2021)Esser, Rombach, and Ommer]{esser2021taming}
Esser, P., Rombach, R., and Ommer, B.
\newblock Taming transformers for high-resolution image synthesis.
\newblock In \emph{Proceedings of the IEEE/CVF conference on computer vision and pattern recognition}, pp.\  12873--12883, 2021.

\bibitem[Feng et~al.(2021)Feng, Lin, Zhu, Zhao, Zhou, and Zha]{feng2021uncertainty}
Feng, R., Lin, Z., Zhu, J., Zhao, D., Zhou, J., and Zha, Z.-J.
\newblock Uncertainty principles of encoding gans.
\newblock In \emph{International Conference on Machine Learning}, pp.\  3240--3251. PMLR, 2021.

\bibitem[Feng et~al.(2022)Feng, He, Fu, Jampani, Akula, Narayana, Basu, Wang, and Wang]{feng2022strucdiffusion}
Feng, W., He, X., Fu, T.-J., Jampani, V., Akula, A., Narayana, P., Basu, S., Wang, X.~E., and Wang, W.~Y.
\newblock Training-free structured diffusion guidance for compositional text-to-image synthesis.
\newblock \emph{arXiv preprint arXiv:2212.05032}, 2022.

\bibitem[Ghadimi \& Lan(2013)Ghadimi and Lan]{ghadimi2013stochastic}
Ghadimi, S. and Lan, G.
\newblock Stochastic first-and zeroth-order methods for nonconvex stochastic programming.
\newblock \emph{SIAM Journal on Optimization}, 23\penalty0 (4):\penalty0 2341--2368, 2013.

\bibitem[Goodfellow et~al.(2014)Goodfellow, Shlens, and Szegedy]{goodfellow2014explaining}
Goodfellow, I.~J., Shlens, J., and Szegedy, C.
\newblock Explaining and harnessing adversarial examples.
\newblock \emph{arXiv preprint arXiv:1412.6572}, 2014.

\bibitem[HaoChen \& Ma(2022)HaoChen and Ma]{haochen2022theoretical}
HaoChen, J.~Z. and Ma, T.
\newblock A theoretical study of inductive biases in contrastive learning.
\newblock \emph{arXiv preprint arXiv:2211.14699}, 2022.

\bibitem[He et~al.(2020)He, Fan, Wu, Xie, and Girshick]{he2020moco}
He, K., Fan, H., Wu, Y., Xie, S., and Girshick, R.
\newblock Momentum contrast for unsupervised visual representation learning.
\newblock In \emph{Proceedings of the IEEE/CVF conference on computer vision and pattern recognition}, pp.\  9729--9738, 2020.

\bibitem[Hertz et~al.(2023)Hertz, Aberman, and Cohen-Or]{hertz2023delta}
Hertz, A., Aberman, K., and Cohen-Or, D.
\newblock Delta denoising score.
\newblock In \emph{Proceedings of the IEEE/CVF International Conference on Computer Vision}, pp.\  2328--2337, 2023.

\bibitem[Hessel et~al.(2021)Hessel, Holtzman, Forbes, Bras, and Choi]{hessel2021clipscore}
Hessel, J., Holtzman, A., Forbes, M., Bras, R.~L., and Choi, Y.
\newblock Clipscore: A reference-free evaluation metric for image captioning.
\newblock \emph{arXiv preprint arXiv:2104.08718}, 2021.

\bibitem[Heusel et~al.(2017)Heusel, Ramsauer, Unterthiner, Nessler, and Hochreiter]{heusel2017gans_fid}
Heusel, M., Ramsauer, H., Unterthiner, T., Nessler, B., and Hochreiter, S.
\newblock Gans trained by a two time-scale update rule converge to a local nash equilibrium.
\newblock \emph{Advances in neural information processing systems}, 30, 2017.

\bibitem[Ho \& Salimans(2022)Ho and Salimans]{ho2022classifier}
Ho, J. and Salimans, T.
\newblock Classifier-free diffusion guidance.
\newblock \emph{arXiv preprint arXiv:2207.12598}, 2022.

\bibitem[Ho et~al.(2020)Ho, Jain, and Abbeel]{ho2020denoising}
Ho, J., Jain, A., and Abbeel, P.
\newblock Denoising diffusion probabilistic models.
\newblock \emph{Advances in Neural Information Processing Systems}, 33:\penalty0 6840--6851, 2020.

\bibitem[Hu et~al.(2022)Hu, Zhili, Zhou, Wang, and Huang]{hu2022your}
Hu, T., Zhili, L., Zhou, F., Wang, W., and Huang, W.
\newblock Your contrastive learning is secretly doing stochastic neighbor embedding.
\newblock \emph{The Eleventh International Conference on Learning Representations}, 2022.

\bibitem[Hu et~al.(2023)Hu, Chen, Wang, Li, Wang, Sun, and Li]{hu2023complexity}
Hu, T., Chen, F., Wang, H., Li, J., Wang, W., Sun, J., and Li, Z.
\newblock Complexity matters: Rethinking the latent space for generative modeling.
\newblock \emph{arXiv preprint arXiv:2307.08283}, 2023.

\bibitem[Huang et~al.(2023)Huang, Sun, Xie, Li, and Liu]{huang2023t2i}
Huang, K., Sun, K., Xie, E., Li, Z., and Liu, X.
\newblock T2i-compbench: A comprehensive benchmark for open-world compositional text-to-image generation.
\newblock \emph{arXiv preprint arXiv:2307.06350}, 2023.

\bibitem[Jia et~al.(2021)Jia, Yang, Xia, Chen, Parekh, Pham, Le, Sung, Li, and Duerig]{jia2021ALIGN}
Jia, C., Yang, Y., Xia, Y., Chen, Y.-T., Parekh, Z., Pham, H., Le, Q., Sung, Y.-H., Li, Z., and Duerig, T.
\newblock Scaling up visual and vision-language representation learning with noisy text supervision.
\newblock In \emph{International conference on machine learning}, pp.\  4904--4916. PMLR, 2021.

\bibitem[Kalantidis et~al.(2020)Kalantidis, Sariyildiz, Pion, Weinzaepfel, and Larlus]{kalantidis2020hardnegforcl}
Kalantidis, Y., Sariyildiz, M.~B., Pion, N., Weinzaepfel, P., and Larlus, D.
\newblock Hard negative mixing for contrastive learning.
\newblock \emph{Advances in Neural Information Processing Systems}, 33:\penalty0 21798--21809, 2020.

\bibitem[Kim et~al.(2023{\natexlab{a}})Kim, Lee, Choi, Jeong, Sohn, and Shin]{kim2023collaborative}
Kim, S., Lee, K., Choi, J.~S., Jeong, J., Sohn, K., and Shin, J.
\newblock Collaborative score distillation for consistent visual synthesis.
\newblock \emph{arXiv preprint arXiv:2307.04787}, 2023{\natexlab{a}}.

\bibitem[Kim et~al.(2023{\natexlab{b}})Kim, Lee, Kim, Ha, and Zhu]{kim2023densediff}
Kim, Y., Lee, J., Kim, J.-H., Ha, J.-W., and Zhu, J.-Y.
\newblock Dense text-to-image generation with attention modulation.
\newblock In \emph{Proceedings of the IEEE/CVF International Conference on Computer Vision}, pp.\  7701--7711, 2023{\natexlab{b}}.

\bibitem[Kovalev \& Gasnikov(2022)Kovalev and Gasnikov]{kovalev2022first}
Kovalev, D. and Gasnikov, A.
\newblock The first optimal acceleration of high-order methods in smooth convex optimization.
\newblock \emph{Advances in Neural Information Processing Systems}, 35:\penalty0 35339--35351, 2022.

\bibitem[Lei et~al.(2019)Lei, Hu, Li, and Tang]{lei2019stochastic}
Lei, Y., Hu, T., Li, G., and Tang, K.
\newblock Stochastic gradient descent for nonconvex learning without bounded gradient assumptions.
\newblock \emph{IEEE transactions on neural networks and learning systems}, 31\penalty0 (10):\penalty0 4394--4400, 2019.

\bibitem[Li et~al.(2019)Li, Chen, Wang, and Carin]{li2019certified}
Li, B., Chen, C., Wang, W., and Carin, L.
\newblock Certified adversarial robustness with additive noise.
\newblock \emph{Advances in neural information processing systems}, 32, 2019.

\bibitem[Li et~al.(2021)Li, Selvaraju, Gotmare, Joty, Xiong, and Hoi]{li2021align}
Li, J., Selvaraju, R., Gotmare, A., Joty, S., Xiong, C., and Hoi, S. C.~H.
\newblock Align before fuse: Vision and language representation learning with momentum distillation.
\newblock \emph{Advances in neural information processing systems}, 34:\penalty0 9694--9705, 2021.

\bibitem[Li et~al.(2022)Li, Li, Xiong, and Hoi]{li2022blip}
Li, J., Li, D., Xiong, C., and Hoi, S.
\newblock Blip: Bootstrapping language-image pre-training for unified vision-language understanding and generation.
\newblock In \emph{International Conference on Machine Learning}, pp.\  12888--12900. PMLR, 2022.

\bibitem[Li et~al.(2023{\natexlab{a}})Li, Li, Savarese, and Hoi]{li2023blip2}
Li, J., Li, D., Savarese, S., and Hoi, S.
\newblock Blip-2: Bootstrapping language-image pre-training with frozen image encoders and large language models.
\newblock \emph{arXiv preprint arXiv:2301.12597}, 2023{\natexlab{a}}.

\bibitem[Li et~al.(2023{\natexlab{b}})Li, Liu, Wu, Mu, Yang, Gao, Li, and Lee]{li2023gligen}
Li, Y., Liu, H., Wu, Q., Mu, F., Yang, J., Gao, J., Li, C., and Lee, Y.~J.
\newblock Gligen: Open-set grounded text-to-image generation.
\newblock In \emph{Proceedings of the IEEE/CVF Conference on Computer Vision and Pattern Recognition}, pp.\  22511--22521, 2023{\natexlab{b}}.

\bibitem[Lian et~al.(2023)Lian, Li, Yala, and Darrell]{lian2023llmground}
Lian, L., Li, B., Yala, A., and Darrell, T.
\newblock Llm-grounded diffusion: Enhancing prompt understanding of text-to-image diffusion models with large language models.
\newblock \emph{arXiv preprint arXiv:2305.13655}, 2023.

\bibitem[Luo et~al.(2023)Luo, Hu, Zhang, Sun, Li, and Zhang]{luo2023diff}
Luo, W., Hu, T., Zhang, S., Sun, J., Li, Z., and Zhang, Z.
\newblock Diff-instruct: A universal approach for transferring knowledge from pre-trained diffusion models.
\newblock \emph{arXiv preprint arXiv:2305.18455}, 2023.

\bibitem[Ma et~al.(2023)Ma, Hu, Wang, and Sun]{ma2023elucidating}
Ma, J., Hu, T., Wang, W., and Sun, J.
\newblock Elucidating the design space of classifier-guided diffusion generation.
\newblock \emph{arXiv preprint arXiv:2310.11311}, 2023.

\bibitem[Mahendran \& Vedaldi(2015)Mahendran and Vedaldi]{mahendran2015imageinvert}
Mahendran, A. and Vedaldi, A.
\newblock Understanding deep image representations by inverting them.
\newblock In \emph{Proceedings of the IEEE conference on computer vision and pattern recognition}, pp.\  5188--5196, 2015.

\bibitem[Malladi et~al.(2024)Malladi, Gao, Nichani, Damian, Lee, Chen, and Arora]{malladi2024fine}
Malladi, S., Gao, T., Nichani, E., Damian, A., Lee, J.~D., Chen, D., and Arora, S.
\newblock Fine-tuning language models with just forward passes.
\newblock \emph{Advances in Neural Information Processing Systems}, 36, 2024.

\bibitem[Nesterov et~al.(2018)]{nesterov2018lectures}
Nesterov, Y. et~al.
\newblock \emph{Lectures on convex optimization}, volume 137.
\newblock Springer, 2018.

\bibitem[Nichol et~al.(2021)Nichol, Dhariwal, Ramesh, Shyam, Mishkin, McGrew, Sutskever, and Chen]{nichol2021glide}
Nichol, A., Dhariwal, P., Ramesh, A., Shyam, P., Mishkin, P., McGrew, B., Sutskever, I., and Chen, M.
\newblock Glide: Towards photorealistic image generation and editing with text-guided diffusion models.
\newblock \emph{arXiv preprint arXiv:2112.10741}, 2021.

\bibitem[OpenAI(2023)]{dalle2}
OpenAI.
\newblock Dalle-2, 2023.
\newblock URL \url{https://openai.com/dall-e-2}.

\bibitem[Patrick et~al.(2021)Patrick, Asano, Kuznetsova, Fong, Henriques, Zweig, and Vedaldi]{patrick2021comptransincl}
Patrick, M., Asano, Y.~M., Kuznetsova, P., Fong, R., Henriques, J.~F., Zweig, G., and Vedaldi, A.
\newblock On compositions of transformations in contrastive self-supervised learning.
\newblock In \emph{Proceedings of the IEEE/CVF International Conference on Computer Vision}, pp.\  9577--9587, 2021.

\bibitem[Phung et~al.(2023)Phung, Ge, and Huang]{phung2023att-refocus}
Phung, Q., Ge, S., and Huang, J.-B.
\newblock Grounded text-to-image synthesis with attention refocusing.
\newblock \emph{arXiv preprint arXiv:2306.05427}, 2023.

\bibitem[Podell et~al.(2023)Podell, English, Lacey, Blattmann, Dockhorn, M{\"u}ller, Penna, and Rombach]{podell2023sdxl}
Podell, D., English, Z., Lacey, K., Blattmann, A., Dockhorn, T., M{\"u}ller, J., Penna, J., and Rombach, R.
\newblock Sdxl: Improving latent diffusion models for high-resolution image synthesis.
\newblock \emph{arXiv preprint arXiv:2307.01952}, 2023.

\bibitem[Poole et~al.(2022)Poole, Jain, Barron, and Mildenhall]{poole2022dreamfusion}
Poole, B., Jain, A., Barron, J.~T., and Mildenhall, B.
\newblock Dreamfusion: Text-to-3d using 2d diffusion.
\newblock \emph{arXiv preprint arXiv:2209.14988}, 2022.

\bibitem[Radford et~al.(2021)Radford, Kim, Hallacy, Ramesh, Goh, Agarwal, Sastry, Askell, Mishkin, Clark, et~al.]{radford2021clip}
Radford, A., Kim, J.~W., Hallacy, C., Ramesh, A., Goh, G., Agarwal, S., Sastry, G., Askell, A., Mishkin, P., Clark, J., et~al.
\newblock Learning transferable visual models from natural language supervision.
\newblock In \emph{International conference on machine learning}, pp.\  8748--8763. PMLR, 2021.

\bibitem[Rombach et~al.(2022)Rombach, Blattmann, Lorenz, Esser, and Ommer]{rombach2022high}
Rombach, R., Blattmann, A., Lorenz, D., Esser, P., and Ommer, B.
\newblock High-resolution image synthesis with latent diffusion models.
\newblock In \emph{Proceedings of the IEEE/CVF Conference on Computer Vision and Pattern Recognition}, pp.\  10684--10695, 2022.

\bibitem[Salman et~al.(2019)Salman, Li, Razenshteyn, Zhang, Zhang, Bubeck, and Yang]{salman2019provably}
Salman, H., Li, J., Razenshteyn, I., Zhang, P., Zhang, H., Bubeck, S., and Yang, G.
\newblock Provably robust deep learning via adversarially trained smoothed classifiers.
\newblock \emph{Advances in Neural Information Processing Systems}, 32, 2019.

\bibitem[Shu et~al.(2021)Shu, Kou, Cao, Wang, and Long]{shu2021zoo}
Shu, Y., Kou, Z., Cao, Z., Wang, J., and Long, M.
\newblock Zoo-tuning: Adaptive transfer from a zoo of models.
\newblock In \emph{International Conference on Machine Learning}, pp.\  9626--9637. PMLR, 2021.

\bibitem[Sohl-Dickstein et~al.(2015)Sohl-Dickstein, Weiss, Maheswaranathan, and Ganguli]{sohl2015deep}
Sohl-Dickstein, J., Weiss, E., Maheswaranathan, N., and Ganguli, S.
\newblock Deep unsupervised learning using nonequilibrium thermodynamics.
\newblock In \emph{International Conference on Machine Learning}, pp.\  2256--2265. PMLR, 2015.

\bibitem[Song et~al.(2020)Song, Sohl-Dickstein, Kingma, Kumar, Ermon, and Poole]{song2020score}
Song, Y., Sohl-Dickstein, J., Kingma, D.~P., Kumar, A., Ermon, S., and Poole, B.
\newblock Score-based generative modeling through stochastic differential equations.
\newblock \emph{arXiv preprint arXiv:2011.13456}, 2020.

\bibitem[Tarvainen \& Valpola(2017)Tarvainen and Valpola]{tarvainen2017mean}
Tarvainen, A. and Valpola, H.
\newblock Mean teachers are better role models: Weight-averaged consistency targets improve semi-supervised deep learning results.
\newblock \emph{Advances in neural information processing systems}, 30, 2017.

\bibitem[Tian et~al.(2020)Tian, Sun, Poole, Krishnan, Schmid, and Isola]{tian2020goodviewforcl}
Tian, Y., Sun, C., Poole, B., Krishnan, D., Schmid, C., and Isola, P.
\newblock What makes for good views for contrastive learning?
\newblock \emph{Advances in neural information processing systems}, 33:\penalty0 6827--6839, 2020.

\bibitem[Wang \& Torr(2022)Wang and Torr]{wang2022traditionalCaG}
Wang, G. and Torr, P.~H.
\newblock Traditional classification neural networks are good generators: They are competitive with ddpms and gans.
\newblock \emph{arXiv preprint arXiv:2211.14794}, 2022.

\bibitem[Wang et~al.(2023{\natexlab{a}})Wang, Fan, Xu, Wang, Mohan, Iandola, Ranjan, Li, Liu, Wang, et~al.]{wang2023steindreamer}
Wang, P., Fan, Z., Xu, D., Wang, D., Mohan, S., Iandola, F., Ranjan, R., Li, Y., Liu, Q., Wang, Z., et~al.
\newblock Steindreamer: Variance reduction for text-to-3d score distillation via stein identity.
\newblock \emph{arXiv preprint arXiv:2401.00604}, 2023{\natexlab{a}}.

\bibitem[Wang et~al.(2023{\natexlab{b}})Wang, Chen, Chen, Ma, Lu, and Lin]{wang2023compositional-attcontrol}
Wang, R., Chen, Z., Chen, C., Ma, J., Lu, H., and Lin, X.
\newblock Compositional text-to-image synthesis with attention map control of diffusion models.
\newblock \emph{arXiv preprint arXiv:2305.13921}, 2023{\natexlab{b}}.

\bibitem[Wang et~al.(2023{\natexlab{c}})Wang, Lu, Wang, Bao, Li, Su, and Zhu]{wang2023prolificdreamer}
Wang, Z., Lu, C., Wang, Y., Bao, F., Li, C., Su, H., and Zhu, J.
\newblock Prolificdreamer: High-fidelity and diverse text-to-3d generation with variational score distillation.
\newblock \emph{arXiv preprint arXiv:2305.16213}, 2023{\natexlab{c}}.

\bibitem[Yin et~al.(2020)Yin, Molchanov, Alvarez, Li, Mallya, Hoiem, Jha, and Kautz]{yin2020deepinversion}
Yin, H., Molchanov, P., Alvarez, J.~M., Li, Z., Mallya, A., Hoiem, D., Jha, N.~K., and Kautz, J.
\newblock Dreaming to distill: Data-free knowledge transfer via deepinversion.
\newblock In \emph{Proceedings of the IEEE/CVF Conference on Computer Vision and Pattern Recognition}, pp.\  8715--8724, 2020.

\end{thebibliography}
\bibliographystyle{icml2024}

%%%%%%%%%%%%%%%%%%%%%%%%%%%%%%%%%%%%%%%%%%%%%%%%%%%%%%%%%%%%
\newpage
\appendix
\onecolumn
\section{Appendix}
\subsection{Importance of discriminative component in conditional generation}
\label{apdix:discri_gener}
We present both quantitative and qualitative results, showing that the discriminative module plays a more central role in aligning with the condition for current text-to-image DPMs.
\textbf{Table} \ref{tab:sd_cfg_vqa} shows the result of the BLIP-VQA score on the \textit{Attribute-binding }sub-dataset of T2I-CompBench of Stable Diffusion v1.5 with different classifier-free-guidance (CFG) scales; clearly, a better alignment of text and image can be achieved by a larger CFG scale.
We further present the images generated by the traditional generation process of Stable Diffusion v1.5 with CFG coefficient 7.5, 30, and the Score Distillation Sampling (SDS) process separately. The results indicate that when strong conditions come along, a smaller CFG also performs poorly, and the higher CFG results in comparative image quality but higher consistency with text.

\begin{table}[H]
\centering
\begin{small}
\caption{The BLIP-VQA score of Stable Diffusion v1.5 with CFG 7.5 and 30 on \textit{Attribute-binding} dataset}
\vspace{0.1in}
\label{tab:sd_cfg_vqa}
\begin{tabular}{l|ccc}
\toprule
 & \textbf{color}  & \textbf{shape}  & \textbf{texture} \\ \midrule
cfg=7.5 & 0.3632 & 0.3490 & 0.3985  \\ 
cfg=30  & 0.4012 & 0.4018 & 0.4386  \\ \bottomrule
\end{tabular} 
\end{small}
\end{table}
\vspace{0.1in}

\begin{figure}[H]
    \centering
    \includegraphics[width=0.95\textwidth]{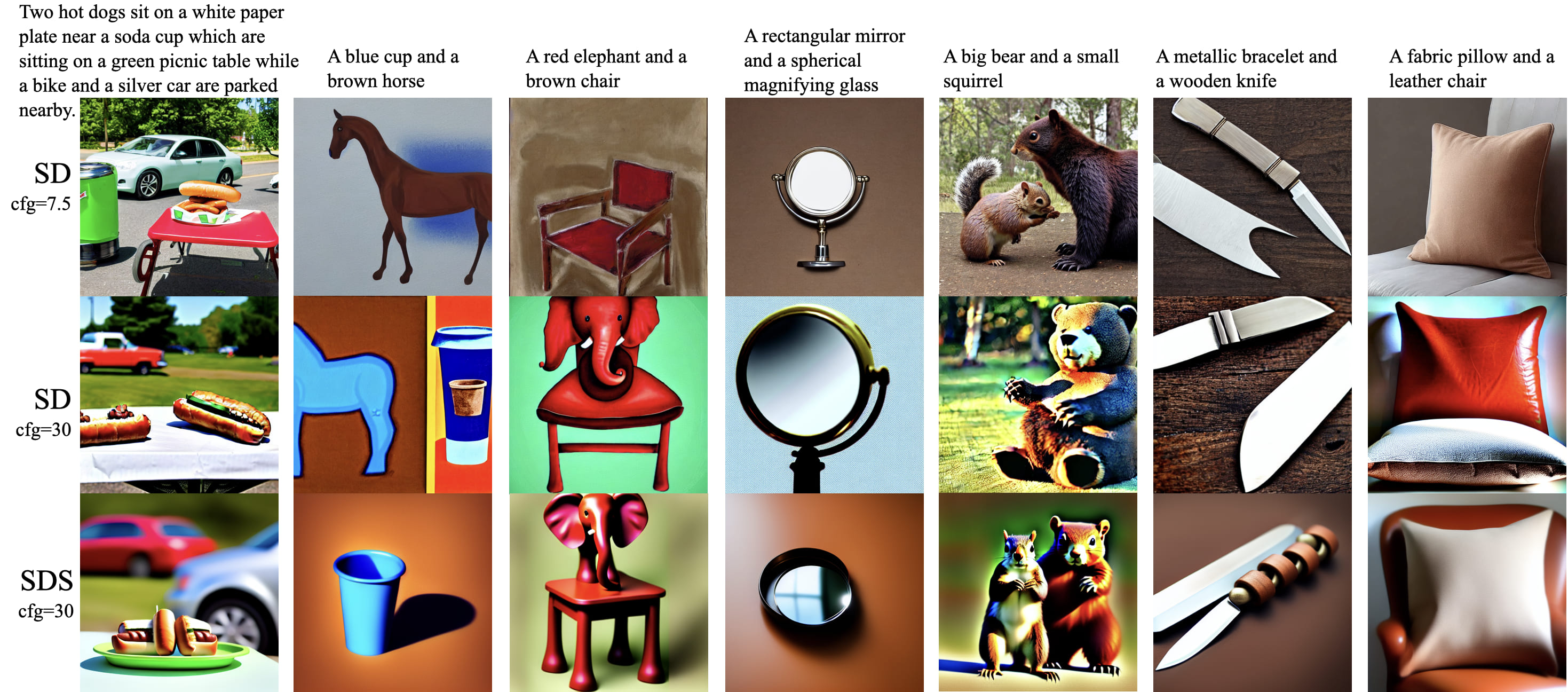}
    \vskip -3mm
    \caption{A comparison of the images generated from the Stable Diffusion (SD) reverse process with CFG scale=7.5, 30 and the Score Distillation Sampling (SDS) process with CFG scale=30.}
    \label{fig:sd_cfg}
\end{figure}
\vskip -5mm

\subsection{Detailed BLIP-2 losses}
We directly utilize the training losses of the BLIP-2 model \citep{li2021align, li2022blip, li2023blip2}, except the image-text contrastive (ITC) loss, because it requires a large batch size that is not suitable for text-to-image generation tasks. The image-text matching loss ($\mathcal{L}_{itm})$ focuses on learning how images and texts align closely. It's a simple binary cross-entropy loss: the model predicts if an image and text pair match or not, using a linear layer called the ITM head based on their combined features. The caption-generation loss ($\mathcal{L}_{cg}$) aims to generate textual descriptions given an image. It optimizes a cross-entropy loss which trains the BERT module to maximize the likelihood of the text in an autoregressive manner. A label smoothing of 0.1 is adopted when computing the loss.

Specifically,
\begin{equation}
    \mathcal{L}_{itm}=\mathbb{E}_{(I, T)} H(\mathbf{y}^{itm}, \mathbf{p}^{itm}(I,T))),
\end{equation}
where $H$ denotes the cross-entropy function, $\mathbf{y}^{itm}$ is a 2-dimensional one-hot vector representing the ground-truth label of whether $I$ and $T$ is a match, $\mathbf{p}^{itm}$ is the output of the ITM head.

\begin{equation}
    \mathcal{L}_{cg}=\mathbb{E}_{(I, \hat{T})} H (\mathbf{y}^{mask}, \mathbf{p}^{mask}(I, \hat{T})),
\end{equation}
where $\hat{T}$ denotes the text with mask, and $\mathbf{p}^{mask}(I, \hat{T})$ denotes the BERT module's output for a masked token, and $\mathbf{y}^{mask}$ is a one-hot vocabulary distribution in which the value of the ground-truth token is set to 1.

\subsection{Exploration on augmentation methods for VLM inversion}
\label{appdx:aug_explore}
Here we show the impact of adding \textit{random resized crop} and \textit{random horizontal flip} augments, and straightforwardly illustrate the potential reasons through BLIP-2 loss that measures the semantic consistency.

\begin{figure}[H]
    \centering
    \includegraphics[width=0.95\textwidth]{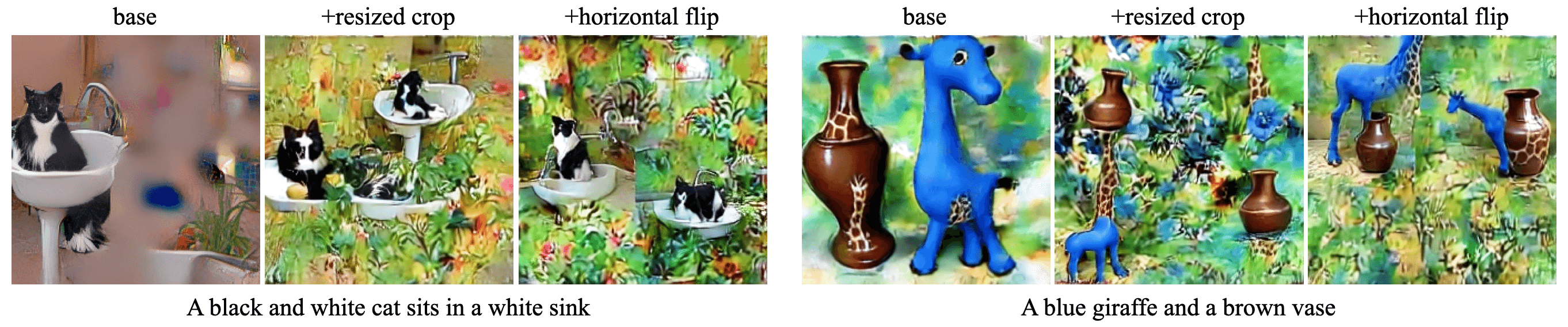}
    \vskip -3mm
    \caption{Images with our base augmentation methods (random affine, random perspective, color-jitter, random erasing, and Gaussian noise), base + random resized crop, and base + random horizontal flip.}
    \label{fig:aug_abla}
\end{figure}
\vskip -5mm

\begin{figure}[H]
    \centering
    \includegraphics[width=0.95\textwidth]{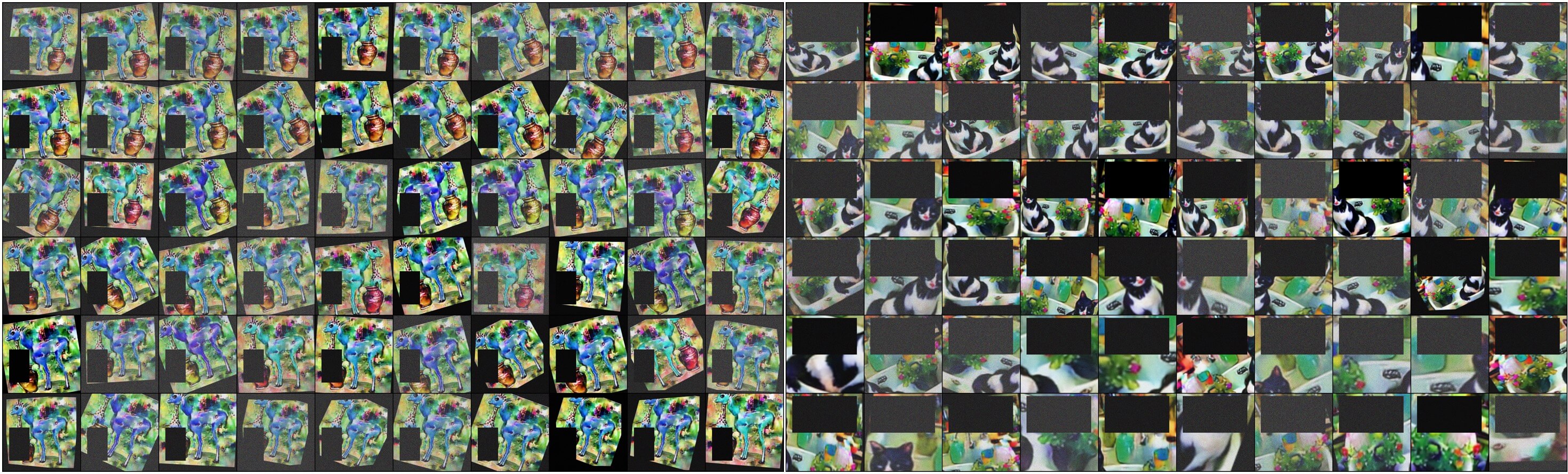}
    \vskip -3mm
    \caption{Sorted augmented images sorted in ascending order of BLIP-2 loss $L(\xb,\yb)$. In the left image, the horizontally flipped images exhibit higher loss. Since all of our captions contain multiple objects, incorporating \textit{horizontal flip} together with \textit{random erasing} easily leads to semantic inconsistencies within the same equivalence class. The right image shows that images augmented through resized cropping have a larger loss, which clearly shows altered semantic information.}
    \label{fig:aug_sort}
\end{figure}
\vskip -5mm

\subsection{More images of our method}
\label{appdix:moreimages}
We show more images generated by our method with the prompts sampled from the attribute-biding, spatial $\&$ non-spatial, and complex sub-datasets of T2I-CompBench, indicating our method's effectiveness in generating images with different types of prompts.

\begin{figure}[H]
    \centering
    \includegraphics[width=0.95\textwidth]{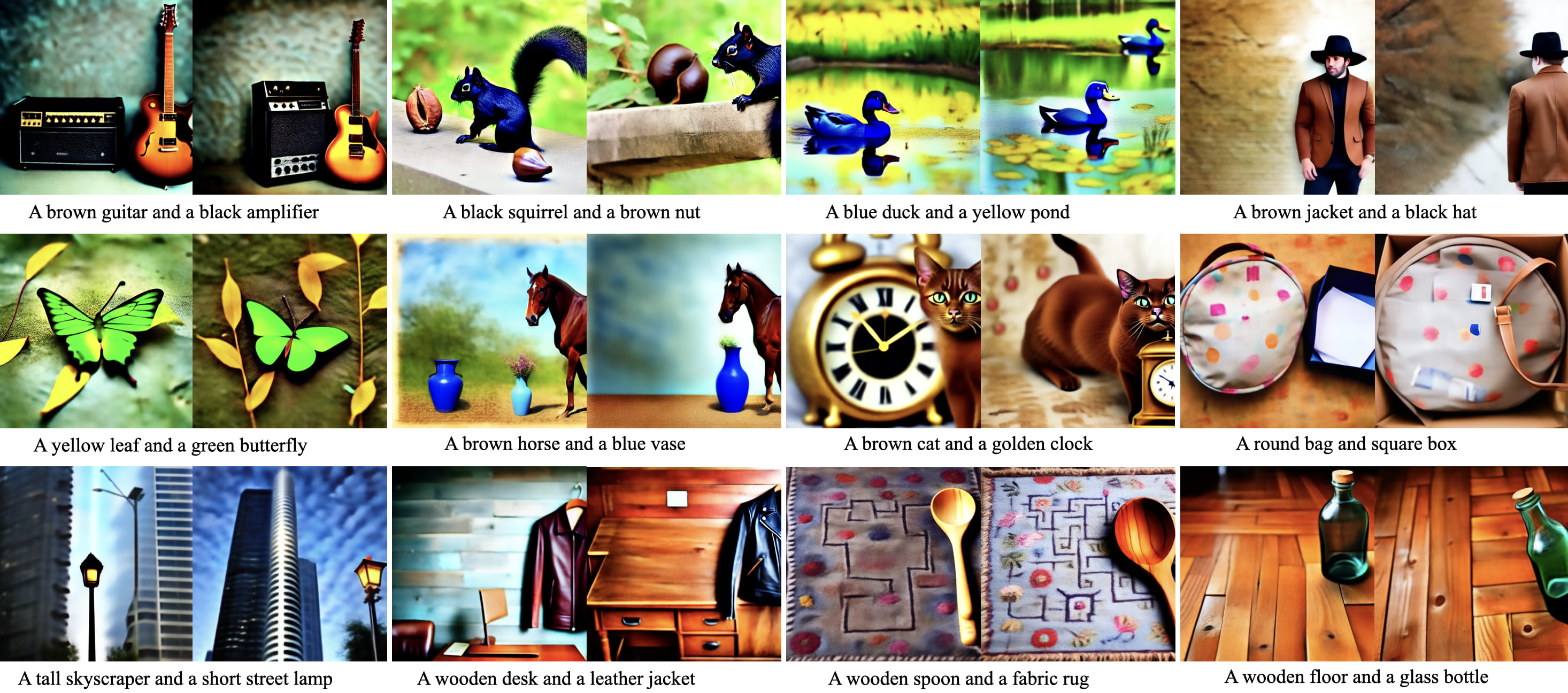}
    \vskip -3mm
    \caption{Images with prompt from \textit{Attribute-binding} sub-dataset}
    \label{fig:cherry_att}
\end{figure}
\vskip -5mm

\begin{figure}[H]
    \centering
    \includegraphics[width=0.95\textwidth]{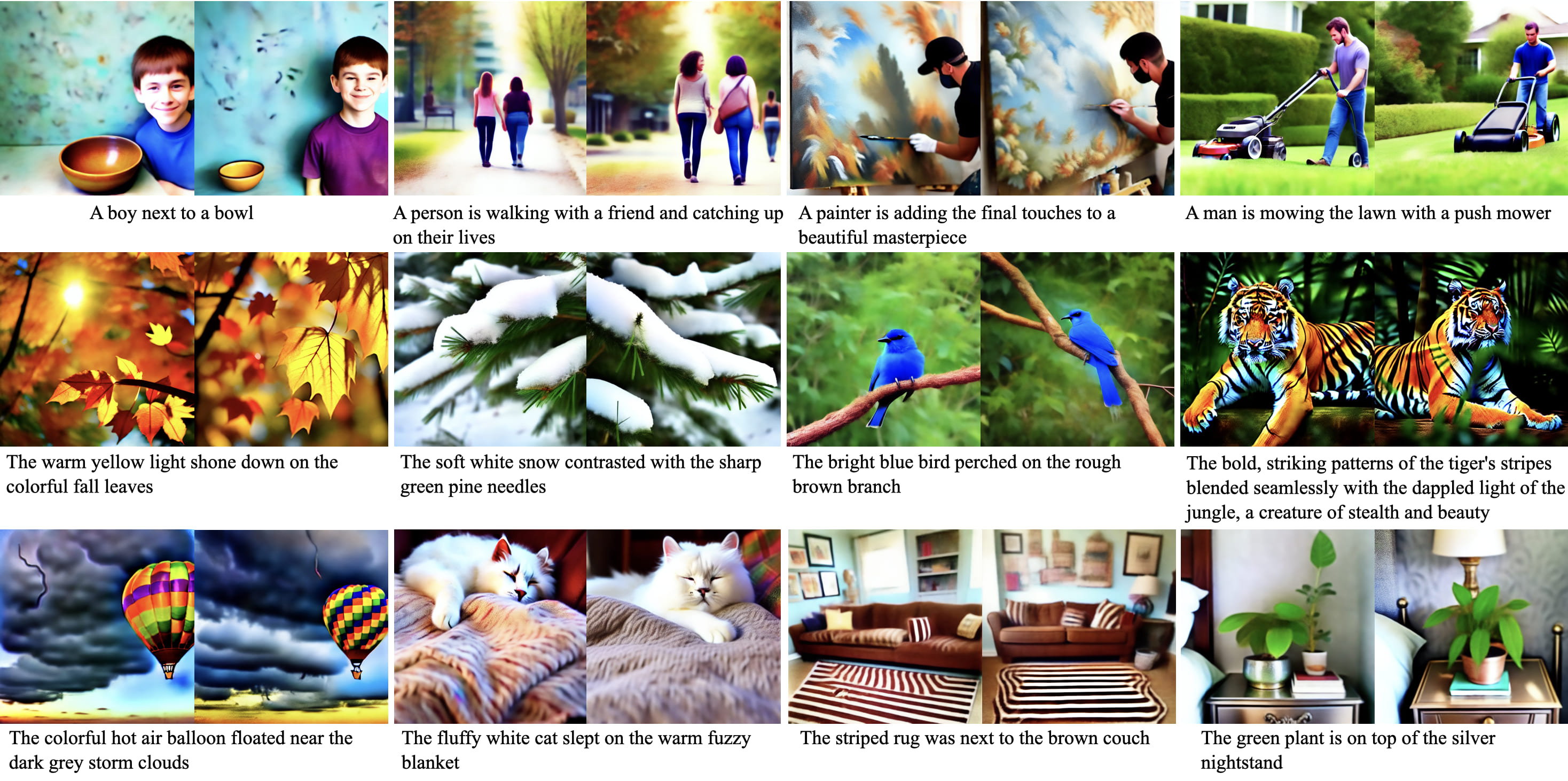}
    \vskip -3mm
    \caption{Images with prompt from \textit{Spatial}, \textit{Non-spatial}, and \textit{Complex} sub-dataset}
    \label{fig:cherry_com}
\end{figure}

\subsection{Role of BLIP-2 inversion and SDS during Optimization}
\label{appdix:blipSDSrole}

\paragraph{Evolution of individual BLIP-2 inversion, SDS, and our method (BLIP-2+SDS)}
To better illustrate the effect of the two components of our method, we present the image evolution during the optimization of the two components work separately and together (\textbf{Figure} \ref{fig:evo1} and \ref{fig:evo2}). Apparently, BLIP-2 inversion can strictly follow the prompt, while SDS tends to partially follow the prompt and neglect some objects contained in the prompt. Besides, for SDS optimization, images can change a lot during evolution, especially when encountering multiple objects in the prompt (see \textbf{Figure} \ref{fig:evo1}), showing that introducing the instruction by cross attention is kind of fragile, the information contained in the prompt is always incompletely captured and unstable. Things are totally different in BLIP-2 inversion; the instruction is well followed, while the perceptual image quality is kind of weak. When these two components work together, better performance is achieved.

\begin{figure}[H]
    \centering
    \includegraphics[width=0.95\textwidth]{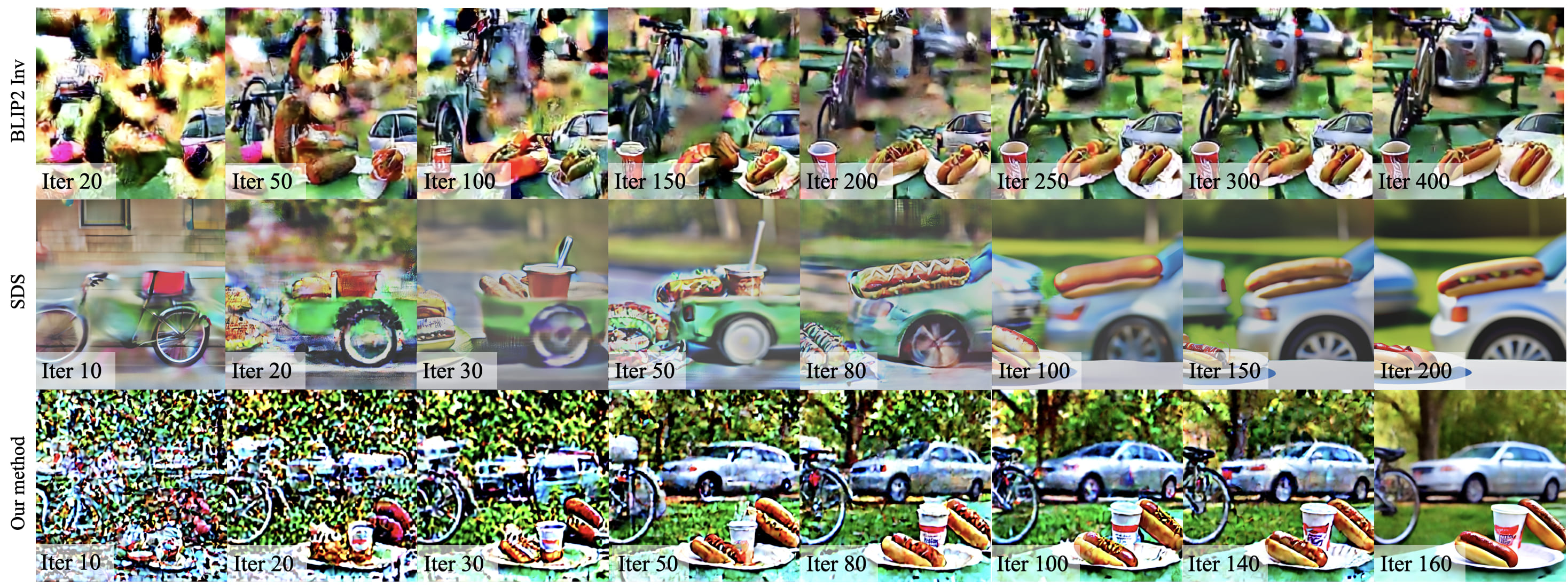}
    \vskip -3mm
    \caption{Image evolution of prompt ``Two hot dogs sit on a white paper plate near a soda cup which are sitting on a green picnic table while a bike and a silver car are parked nearby".}
    \label{fig:evo1}
\end{figure}
\vskip -5mm

\begin{figure}[H]
    \centering
    \includegraphics[width=0.95\textwidth]{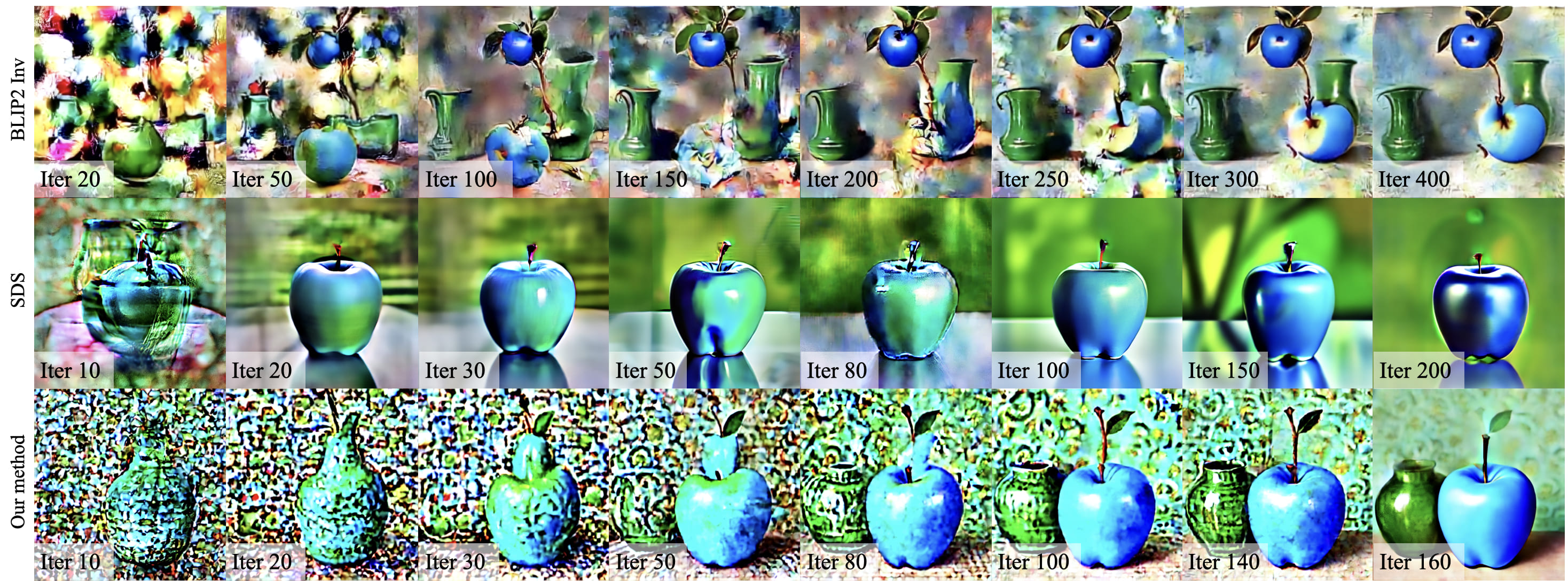}
    \vskip -3mm
    \caption{Images evolution of prompt ``A blue apple and a green vase".}
    \label{fig:evo2}
\end{figure}

\paragraph{Visualization of the gradient of BLIP-2 and SDS}
To gain a more intuitive understanding of the distinct roles of the two components in our method, we visualized the gradients provided by each module (see \textbf{Figure} \ref{fig:grad}). Our approach to visualization is quite straightforward. We directly selected the first three (out of four) channels and then presented them in the form of images. It is evident that the gradient information from BLIP-2 is primarily concentrated on the objects, especially those with incorrectly generated attributes or those that are omitted. On the other hand, the gradients provided by SDS are more comprehensive, clearly outlining the contours of objects while also optimizing the background. Although our visualization method is direct and somewhat rudimentary, it effectively highlights the distinct functions of each module. For example, with the prompt ``a blue apple and a green vase", the SDS gradient reveals the outlines of two apples, while the BLIP-2 gradient focuses on one of the apples with the wrong attribute, attempting to transform it into a vase that matches the text description (see \textbf{Figure} \ref{fig:grad}). 

% This dynamic can lead to a competitive interplay between the two components. 

\begin{figure}[H]
\begin{center}
    \subfigure[A blue apple and a green vase]{
    \begin{minipage}[t]{0.45\linewidth}
            \centering
            \includegraphics[width=1\linewidth]{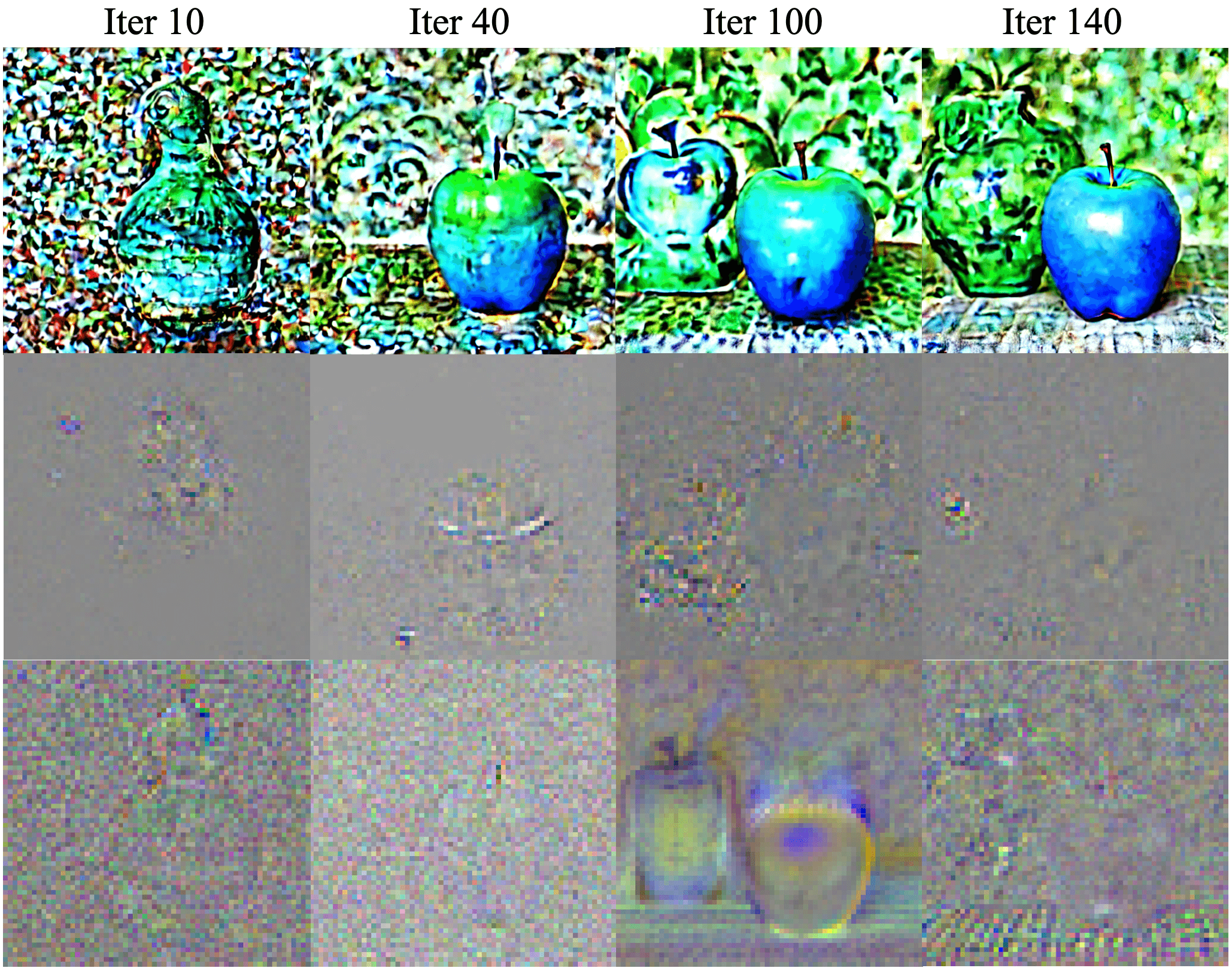}
        \end{minipage}   
    }
    \subfigure[A blue backpack and a red book]{
    \begin{minipage}[t]{0.45\linewidth}
            \centering
            \includegraphics[width=1\linewidth]{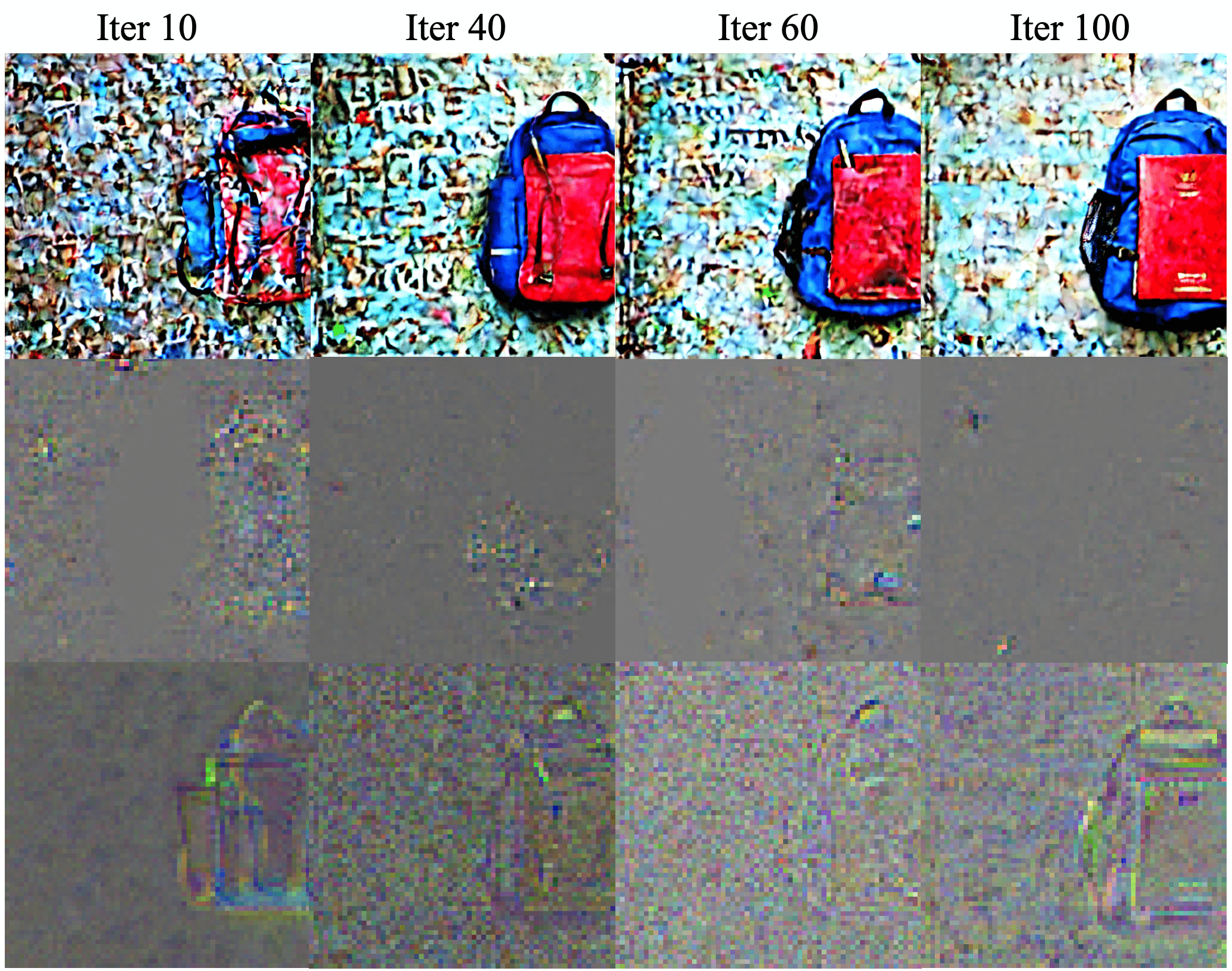}
        \end{minipage}
    }
\end{center}
\vskip -3mm
\caption{Visualization of the Gradient of BLIP-2 and SDS. The images, the gradient of BLIP-2, and the gradient of SDS visualization are separately shown from top to bottom row.}
\label{fig:grad}
\end{figure}

\subsection{Limitations of our method}
Although our method is simple and effective, it still has certain shortcomings (see \textbf{Figure} \ref{fig:fail}). First, our method struggles to generate images with accurate positional information. We attribute this to BLIP-2's insensitivity to capturing location details. Individual BLIP-2 inversion, despite achieving highly accurate attribute matching, performs poorly in the \textit{Spatial} subset (see BLIP-2 Inv result in \textbf{Table} \ref{tab:overallresult}). Additionally, our method often generates images that directly incorporate the prompt. This may be due to the direct involvement of a BERT text encoder in the pre-trained BLIP-2 contained in our current approach. These issues are also encountered in DALLE3. Lastly, as our method does not directly utilize a specialized image generation model, it sometimes produces images that are not entirely realistic. For instance, to achieve correct attribute matching, our method may generate images where multiple objects are merged into one, which seems quite unrealistic.

\begin{figure}[H]
    \centering
    \includegraphics[width=0.95\textwidth]{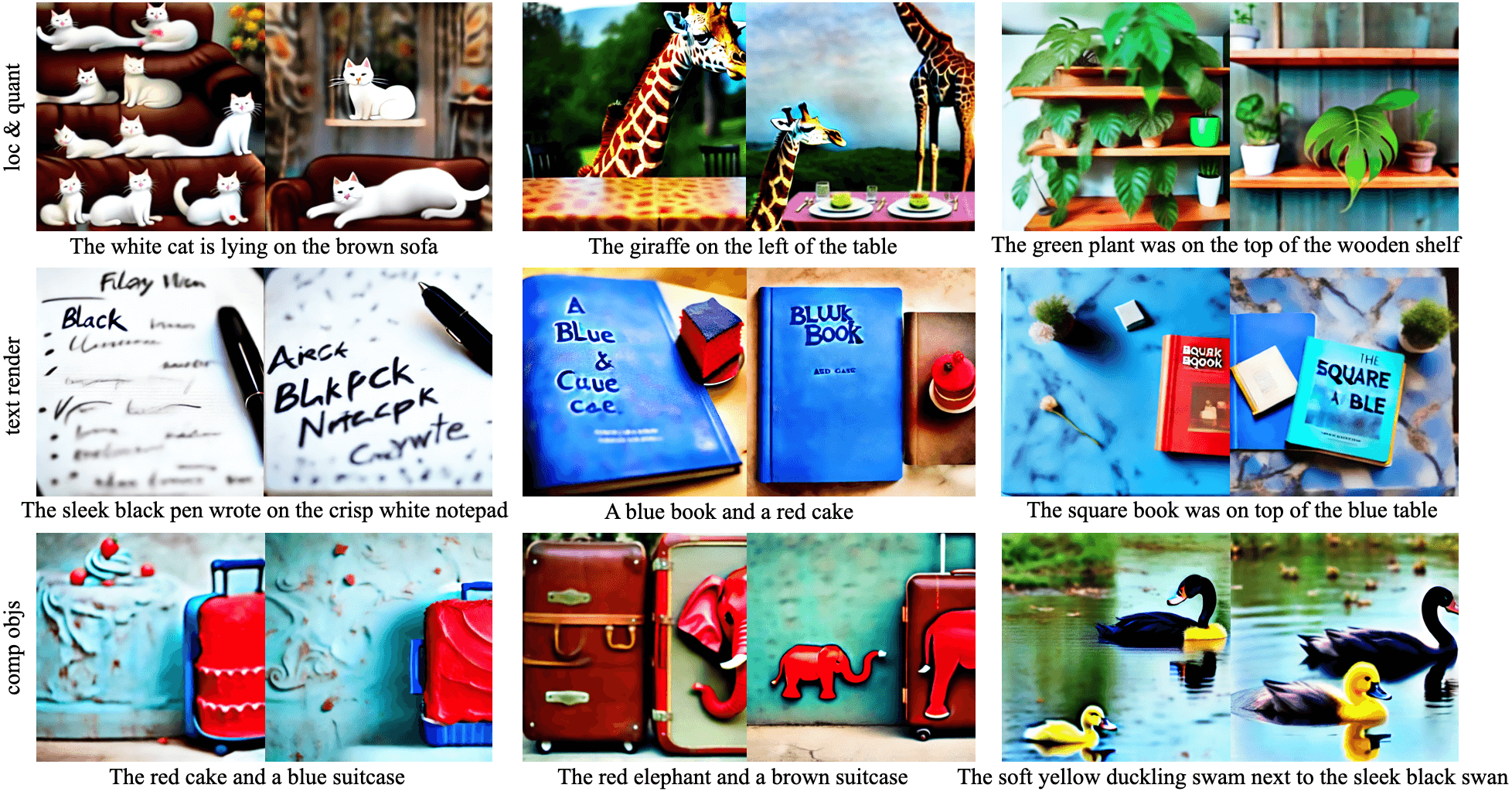}
    \vskip -3mm
    \caption{Limitations of our method.}
    \label{fig:fail}
\end{figure}

\section{Technical Details}
\label{sec:app_tech}

We introduce two propositions, both of which state that the smoothed version of $f$ by taking the expectation of the noise $\mathbb{E}_{\epsilon\sim P} f(x+\epsilon)$ has nicer property. Proposition \ref{prop1} states that the smoothed function has a higher smoothness (roughly speaking, it can take a higher degree of derivatives), and Proposition \ref{prop2} states that the smoothed function has a smaller condition number, which may lead to faster convergence rate when applying gradient descent for minimization.

Before introducing these two propositions, we state some necessary terminologies. We say a function $f$ has smoothness $m$ if $f$ is in the Sobolev space $H^m(\mathbb{R}^d)$. The Fourier transform of $f\in L_1(\mathbb{R}^d)$ is given by $$\mathcal{F}(f)(\omega)=(2\pi)^{-d/2}\int_{\mathbb{R}^d} f(x) e^{-i x^T\omega}d x.$$ The characteristic function of a random variable $\epsilon$ following distribution $P$ is defined by $\varphi(t) = \mathbb{E}_{\epsilon\sim P} e^{i\epsilon^\top t}$. Let $\partial^2 f(x)$ be the Hessian matrix on $x\in \mathbb{R}^d$. For two positive matrices $A,B$, we write $A\preceq B$ (or $B\succeq A$) if $B-A$ is semi-positive definite.

\begin{proposition}[Improved smoothness]\label{prop1}
    Let $f\in H^m(\mathbb{R}^d)$. Assume that the characteristic function of the noise $\epsilon$ with mean zero satisfies
    \begin{align*}
        \varphi(t) \leq c_1(1+\|\omega\|^2)^{-m_1/2}, \forall \omega \in \mathbb{R}^d.
    \end{align*}
Then $g(x)\coloneqq\mathbb{E}_\epsilon f(x+\epsilon)$ has smoothness at least $m+m_1$, i.e., $g\in H^{m+m_1}(\mathbb{R}^d)$. Furthermore, if $\epsilon$ is Gaussian, then $\mathbb{E}_\epsilon f(x+\epsilon)$ has infinite smoothness, i.e., for any $m_2>0$, $g \in H^{m+m_2}(\mathbb{R}^d)$. Moreover, $\tilde{L}_y$  is $\sqrt{\frac{2 R^2}{\pi \sigma^2}}$-Lipschitz when $\epsilon \sim \mathcal{N}(0, \sigma^2 I)$ and $L_y: \RR^d \mapsto \RR$ with $R\ge \sup_{x \in \RR^d} |L_y(x)|$.
\end{proposition}

\begin{proof}
    Since $f\in H^m(\mathbb{R}^d)$, it can be shown that \cite{adams2003sobolev}
    \begin{align*}
        \int_{\mathbb{R}^d} |\mathcal{F}(f)(\omega)|^2(1+\|\omega\|_2^2)^{m}{\rm d}\omega < \infty.
    \end{align*}
    The Fourier inversion theorem implies that
    \begin{align*}
        g(x) = \mathbb{E}_\epsilon f(x+\epsilon) = & \mathbb{E}_\epsilon \int_{\RR^d}\mathcal{F}(f)(\omega)e^{i(x+\epsilon)^\top \omega} {\rm d}\omega\nonumber\\
        = & \mathbb{E}_\epsilon \int_{\RR^d}\mathcal{F}(f)(\omega)e^{ix^\top \omega} e^{i\epsilon^\top \omega}{\rm d}\omega\nonumber\\
        = & \int_{\RR^d}\mathcal{F}(f)(\omega)e^{ix^\top \omega} \varphi(\omega){\rm d}\omega.
    \end{align*}
    Therefore, we have 
    \begin{align*}
        & \int_{\mathbb{R}^d} |\mathcal{F}(g)(\omega)|^2(1+\|\omega\|_2^2)^{m+m_1}{\rm d}\omega\nonumber\\
        = & \int_{\mathbb{R}^d} |\mathcal{F}(f)(\omega)|^2|\varphi(\omega)|^2(1+\|\omega\|_2^2)^{m+m_1}{\rm d}\omega\nonumber\\
        \leq & c_1^2\int_{\mathbb{R}^d} |\mathcal{F}(f)(\omega)|^2(1+\|\omega\|_2^2)^{m}{\rm d}\omega < \infty,
    \end{align*}
    which implies $g\in H^{m+m_1}(\mathbb{R}^d)$.

    Similarly, if $\epsilon$ is Gaussian, we have $\varphi(\omega) = e^{-\frac{1}{2}\sigma^2\omega^2}$. Hence,  for any $m_2>d/2$, it holds that 
    \begin{align*}
        & \int_{\mathbb{R}^d} |\mathcal{F}(g)(\omega)|^2(1+\|\omega\|_2^2)^{m+m_2}{\rm d}\omega\nonumber\\
        = & \int_{\mathbb{R}^d} |\mathcal{F}(f)(\omega)|^2e^{-\sigma^2\omega^2}(1+\|\omega\|_2^2)^{m+m_1}{\rm d}\omega\nonumber\\
        \leq & C\int_{\mathbb{R}^d} |\mathcal{F}(f)(\omega)|^2(1+\|\omega\|_2^2)^{m}{\rm d}\omega < \infty,
    \end{align*}
    for some constant $C>0$, which implies $g\in H^{m+m_2}(\mathbb{R}^d)$.

The proof of Lipschitz continuity follows from known proofs. Here, we adopt and expand the proof of Lemma 1 of \citep{salman2019provably}, which shows that $\tilde{L}_y$ is $\sqrt{\frac{2}{\pi}}$-Lipschitz when $\epsilon \sim \mathcal{N}(0, I)$ and $L_y: \RR^d \mapsto \RR$ with $L_y(x) \in [0,1]$. By the mean value theorem, there exists some $c$ interpolating $x$ and $x'$ such that
$$
\|\tilde{L}_y(x) - \tilde{L}_y(x')\| = \|\nabla \tilde{L}_y(c)^\top(x'-x)\| =  \|\nabla \tilde{L}_y(c)^\top \frac{(x'-x)}{\|(x'-x)\|}\|(x'-x)\|\| \le | \tilde{L}_y(c)^\top v| \|(x'-x)\|.
$$
where $v=\frac{(x'-x)}{\|(x'-x)\|}$ is a unit vector. Thus, the desired statement holds if $|\nabla \tilde{L}_y(c)^\top v| \le \sqrt{\frac{2 R^2}{\pi \sigma^2}}$ for any unit vector $v$. With $\epsilon=t-x$ (and the symmetry of the Gaussian), we have $\tilde{L}_y(x) = \frac{1}{\sigma^d (2\pi)^{d/2}} \int_{\RR^d} L_y(t) \exp\left(-\frac{1}{2\sigma^2} \|x-t\|^2\right) dt$ and hence 
$$
\nabla \tilde{L}_y(x) = \frac{1}{\sigma^d (2\pi)^{d/2}} \int_{\RR^d} L_y(t) \frac{1}{\sigma^2} (t-x) \exp\left(-\frac{1}{2\sigma^2} \|x-t\|^2\right) dt.
$$
Thus, by using $R\ge \sup_{x \in \RR^d} |L_y(x)|$ and classical integration of the Gaussian density (e.g., see \citealp{salman2019provably}), 
\begin{align*}
|\nabla \tilde{L}_y(x)^\top v| & \le \frac{R}{\sigma^2 \sigma^d (2\pi)^{d/2}} \int_{\RR^d} |(x-t)^\top v| \exp\left(-\frac{1}{2\sigma^2} \|x-t\|^2\right) dt
\\ & = \frac{R}{\sigma^2 \sigma \sqrt{2\pi}} \int_{-\infty}^{\infty} |s| \exp\left(-\frac{s^2}{2\sigma^2} \right) ds
\\ & = \sqrt{\frac{2 R^2}{\pi \sigma^2}}.
\end{align*}
where the last line follows from the fact that $\int_{-\infty}^{\infty} |s| \exp\left(-\frac{s^2}{2 \sigma^2}\right) ds = 2 \sigma^2$.
\end{proof}
A smaller Lipschitz constant can benefit optimization because the convergence speed tends to decrease as the Lipschitz constant increases, where typically the learning rate also needs to be smaller to achieve the convergence rate as the Lipschitz constant increases
\citep{ghadimi2013stochastic, bertsekas2016nonlinear}. 
For example, the convergence rate degrades (more than quadratically) as the Lipschitz constant increases (linearly) in Theorem 3 of \citep{lei2019stochastic} for SGD on nonconvex optimization (since their constant $C$ depends on the Lipschitz constant in its proof). Therefore, the smoothness gained by the random augmentation can benefit optimization.

\begin{proposition}[Improved condition number of convex optimization]\label{prop2}
    Let $f(x)$ be a convex function with continuous $\partial^2 f(x)$ and
    \begin{align*}
        \mu I \preceq \partial^2 f(x) \preceq L I, \forall x\in \mathbb{R}^d,
    \end{align*}
    where $I$ is an identity matrix and $\mu, L$ are two positive constants. Assume $\epsilon\sim P$ has mean zero and support $\mathbb{R}^d$, and $\partial^2\mathbb{E}_\epsilon f(x+\epsilon) = \mathbb{E}_\epsilon \partial^2 f(x+\epsilon)$. 
    Then 
    \begin{align*}
        \mu_1 I \preceq \partial^2 g(x) \preceq L_1 I, \forall x\in \mathbb{R}^d,
    \end{align*}
    with $\mu_1\geq \mu$ and $L_1\leq L$, where $g(x) = \mathbb{E}_\epsilon f(x+\epsilon)$. In addition, the inequalities are strict, i.e., $\mu_1 > \mu$ and $L_1 < L$, if there exist $x_0,x_1\in \mathbb{R}^d$ and constants $c_1,c_2>0$ such that 
    \begin{align}\label{eq_prop2}
        \partial^2 f(x_0) \succeq (\mu+c_1)I, \partial^2 f(x_1) \preceq (L-c_2)I.
    \end{align}
\end{proposition}

\begin{remark}
    Theorem 2.2.14 of \citet{nesterov2018lectures} states that an upper bound of the gradient descent is 
    \begin{align*}
        \|x_k - x^*\| \leq \left(1-\frac{2}{1+\kappa}\right)^k\|x_0 - x^*\|,
    \end{align*}
    where $\kappa=L/\mu$ is the condition number. A larger $\kappa$ indicates a slower convergence rate. Proposition \ref{prop2} states that for the smoothed version of $f$, it can have a smaller condition number, thus indicating a faster convergence rate for gradient descent.
\end{remark}

\begin{proof}
    Direct computation shows that 
    \begin{align*}
        \partial^2 g(x) = \mathbb{E}_\epsilon \partial^2 f(x+\epsilon) = \int_{\RR^d} \partial^2 f(x+\epsilon){\rm d}P,
    \end{align*}
    which clearly satisfies 
    \begin{align*}
        \mu_1 I \preceq \partial^2 g(x) \preceq L_1 I, \forall x\in \mathbb{R}^d,
    \end{align*}
    with $\mu_1\geq \mu$ and $L_1\leq L$, since $\int_{\RR^d}{\rm d}P = 1$.

    Suppose \eqref{eq_prop2} holds. We only show $\mu_1>\mu$, since $L_1<L$ can be shown similarly. By the continuity of $\partial^2 f(x)$, there exists a neighbor hood $B(x_0,\delta)$ with $\delta>0$ such that $\partial^2 f(x_0) \succeq (\mu+c_1')I$ with some positive constant $c_1'>0$. Hence, 
    \begin{align*}
        \partial^2 g(x) = & \int_{\RR^d} \partial^2 f(x+\epsilon){\rm d}P\nonumber\\
        = & \int_{\RR^d\backslash B(x_0,\delta)} \partial^2 f(x+\epsilon){\rm d}P + \int_{B(x_0,\delta)} \partial^2 f(x+\epsilon){\rm d}P\nonumber\\
        \succeq & \mu I \int_{\RR^d\backslash B(x_0,\delta)} {\rm d}P + (\mu+c_1')I \int_{B(x_0,\delta)} {\rm d}P,
    \end{align*}
    which implies $\mu_1>\mu$, since $\int_{\RR^d}{\rm d}P = 1$ and $\int_{B(x_0,\delta)}{\rm d}P > 0$. Similarly, $L_1<L$. This finishes the proof.
\end{proof}

\section{Experiment Details}
\label{app:exp_detail}
We optimize the randomly initialized $\zb$ for 160 iterations. In the first 150 iterations, $\zb$ is firstly updated using gradients provided by $\mathcal{L}_{align}$ backpropagation via an Adam optimizer. Subsequently, it is further refined using gradients provided by SDS through an SGD optimizer without momentum. 
The norm of the gradient from BLIP-2 is always twice the gradient from SDS; a noise $\epsilon_z \sim \mathcal{N}(0, 0.04)$ is added on the updated $\zb$. 
In the last 10 iterations, $\zb$ is updated solely based on the gradients from SDS. We initialize the learning rate at 1.0, which then gradually diminishes to 0.5 following a cosine decay schedule.  The SDS weight gradually decays from 800 to 400 following a cosine decay rate. We also implement exponential moving average (EMA) restart at 40 and 100 iterations. All experiments are conducted on the Tesla V100 gpus.

\subsection{Abation results}
\label{appdix:moreablation}
The results of ablation explorations on the constant SDS weight and random noise scale are represented in \textbf{Table} \ref{table:sds_cfg} and \ref{tab_noise}.

\begin{table*}[h]
\caption{Result of different CFG-scale and SDS weight $w_2$}
\label{table:sds_cfg}
\vskip 0.1in
\centering
\begin{sc}
\begin{small}
\begin{tabular}{@{}c|cccc@{}}
\toprule
{cfg / $w_2$} & 800 & 1000 & 1500 & 2000 \\ \midrule
10 & \begin{tabular}[c]{@{}c@{}}56.52\\ (70.28/35.98/63.32)\end{tabular} & \begin{tabular}[c]{@{}c@{}}57.70\\ (73.87/36.24/64.83)\end{tabular} & \begin{tabular}[c]{@{}c@{}}54.40\\ (67.60/35.14/63.47)\end{tabular} & \begin{tabular}[c]{@{}c@{}}52.62\\ (62.58/32.85/62.43)\end{tabular} \\ \midrule
20 & \begin{tabular}[c]{@{}c@{}}\textbf{62.74}\\ (\textbf{74.28/40.78/73.21})\end{tabular} & \begin{tabular}[c]{@{}c@{}}\textbf{61.08}\\ (\textbf{71.06/39.59/72.58})\end{tabular} & \begin{tabular}[c]{@{}c@{}}60.39\\ (67.86/40.05/72.26)\end{tabular} & \begin{tabular}[c]{@{}c@{}}58.06\\ (68.34/39.19/66.67)\end{tabular} \\ \midrule
30 & \begin{tabular}[c]{@{}c@{}}\textbf{62.18}\\ (\textbf{73.32/41.01/74.12})\end{tabular} & \begin{tabular}[c]{@{}c@{}}\textbf{60.90}\\ (\textbf{72.36/40.19/70.50})\end{tabular} & \begin{tabular}[c]{@{}c@{}}59.78\\ (66.98/42.64/69.72)\end{tabular} & \begin{tabular}[c]{@{}c@{}}56.70\\ (66.93/39.59/66.59)\end{tabular} \\ \bottomrule
\end{tabular}
\end{small}
\end{sc}
\end{table*}

\begin{table*}[h]
\caption{Results of adding random noise on $\textbf{z}$}
\label{tab_noise}
\centering
\begin{small}
\begin{tabular}{c|cccc}
\hline
$\sigma$ & 0.05 & 0.1 & 0.2& 0.3 \\ \hline
vqaScore & \begin{tabular}[c]{@{}c@{}}68.21\\ (79.02/46.50/79.07)\end{tabular} & \begin{tabular}[c]{@{}c@{}}70.88\\ (81.12/51.62/80.00)\end{tabular} & \begin{tabular}[c]{@{}c@{}}\textbf{71.04}\\ (\textbf{79.54/51.35/83.23})\end{tabular} & \begin{tabular}[c]{@{}c@{}}64.57\\ (70.54/43.66/79.50)\end{tabular} \\ \hline
\end{tabular}
\end{small}
\end{table*}

\end{document}